\documentclass[lettersize,journal]{IEEEtran}
\usepackage{array}
\usepackage[caption=false,font=normalsize,labelfont=sf,textfont=sf]{subfig}
\usepackage{textcomp}
\usepackage{stfloats}
\usepackage{url}
\usepackage{verbatim}
\usepackage{graphicx}
\usepackage{cite}

\usepackage[ruled,noend,linesnumbered,lined]{algorithm2e}
\usepackage{bm}
\usepackage{multirow}
\usepackage{amsmath,amssymb,amsfonts}
\usepackage{amsthm}
\newtheorem{theorem}{\bf \emph{Theorem}}
\newtheorem{lemma}{\bf \emph{Lemma}}
\newtheorem{definition}{\bf \emph{Definition}}
\newtheorem{example}{\bf \emph{Example}}
\newtheorem{remark}{\bf \emph{Remark}}
\usepackage{color}
\usepackage{textgreek}
\usepackage{hyperref}

\usepackage{threeparttable}
\usepackage{makecell,multirow,diagbox}
\usepackage{rotating}



\begin{document}

\title{Feature Explosion: a generic optimization strategy for outlier detection algorithms}

\author{Qi~Li
\IEEEcompsocitemizethanks{\IEEEcompsocthanksitem Q. Li is with School of Information Science and Technology, Beijing Forestry University, Beijing, 100083, China.\protect\\
E-mail: liqi2024@bjfu.edu.cn}
}

\maketitle

\begin{abstract}
Outlier detection tasks aim at discovering potential issues or opportunities and are widely used in cybersecurity, financial security, industrial inspection, \emph{etc}. To date, thousands of outlier detection algorithms have been proposed. Clearly, in real-world scenarios, such a large number of algorithms is unnecessary. In other words, a large number of outlier detection algorithms are redundant. We believe the root cause of this redundancy lies in the current highly customized (\emph{i.e.}, non-generic) optimization strategies. Specifically, when researchers seek to improve the performance of existing outlier detection algorithms, they have to design separate optimized versions tailored to the principles of each algorithm, leading to an ever-growing number of outlier detection algorithms. To address this issue, in this paper, we introduce the ‘explosion’ from physics into the outlier detection task and propose a ‘generic’ optimization strategy based on ‘feature explosion’, called OSD (\underline{\textbf{O}}ptimization \underline{\textbf{S}}trategy for outlier \underline{\textbf{D}}etection algorithms). In the future, when improving the performance of existing outlier detection algorithms, it will be sufficient to invoke the OSD plugin without the need to design customized optimized versions for them. We compared the performances of 14 outlier detection algorithms on 24 datasets before and after invoking the OSD plugin. The experimental results show that the performances of all outlier detection algorithms are improved on almost all datasets. In terms of average accuracy, OSD make these outlier detection algorithms improve by 15\% (AUC), 63.7\% (AP).
\end{abstract}


\section{Introduction}
\label{sec:introduction}
\textbf{Significance and Challenges.} Outlier detection is a key task in data analytics and machine learning, which aims to detect outliers in large amounts of data that do not conform to normal patterns. These outliers often represent potential issues or opportunities, such as hacking in computer networks, fraud in financial systems, and equipment failures in industrial production\cite{panjei2022survey}. However, the current outlier detection task faces two challenges:

\underline{$\bullet$ Challenge 1 (Non-generic Optimization Strategies):} To date, several thousand outlier detection algorithms have been proposed. However, in real-world scenarios, such a large number of algorithms is unnecessary. That is, most of the existing outlier detection algorithms are redundant. We believe that the root cause of redundancy is that too many optimized-version algorithms (see Definition \ref{def:optimized-version}) have been proposed but their optimization strategies are not generic. For example, KNNLOF\cite{xu2022outlier} is an optimized-version algorithm of LOF \cite{breunig2000lof} (a classical algorithm that detects outliers by exploiting local density differences among neighbors), and it proposes a neighbor querying method for different density distributions to improve the performance of LOF; DIF \cite{xu2023deep} is an optimized-version algorithm of IForest \cite{liu2012isolation} (anther classical outlier detection algorithm that partitions feature to generate a tree structure and then detects outliers based on the object's position in the tree structure), and it designs a non-linear partitioning based on the deep learning for better detection of hard outliers in complex datasets. However, since the principle of IForest only partitions features to generate a tree structure without querying neighbors, the optimization strategy of KNNLOF cannot be applied to IForest; Since the principle of LOF only calculates the density difference between neighbors without the need to partition features, the optimization strategy of DIF cannot be applied to LOF either. As a result, when seeking to improve the performances of LOF and IForest, researchers have to design different optimized-version algorithms for LOF and IForest, leading to an ever-growing number of outlier detection algorithms. Obviously, if a \textbf{‘generic’} optimization strategy \textbf{applicable to various outlier detection principles} is proposed, researchers will no longer need to design different optimized-version algorithms for existing outlier detection algorithms, but only need to call the generic optimization strategy plugin, thereby curbing the algorithm redundancy in outlier detection task.

\begin{definition}
\label{def:optimized-version}
\textbf{(The optimized-version algorithm)} For an outlier detection algorithm $\mathcal{A}$, by changing its principle, $\mathcal{A}$ becomes another outlier detection algorithm $\mathbb{A}$. If $\mathbb{A}$ outperforms $\mathcal{A}$, then $\mathbb{A}$ is an optimized-version algorithm of $\mathcal{A}$.
\end{definition}

\underline{$\bullet$ Challenge 2 (Loss of Original Advantages):} It is well-known that no outlier detection principle is flawless. Any optimized-version algorithm, which addresses original algorithm's some shortcomings, will inevitably introduce new limitations \cite{yang2024generalized}. For example, IForest is initially insensitive to data scale, but its optimized-version algorithm, DIF, becomes unsuitable for small-scale datasets due to the incorporation of deep network architectures. How to improve the performance of outlier detection algorithms while retaining their original advantages is another challenge in the current outlier detection task. Obviously, solving this challenge is of great practical significance.

\textbf{Ideas and Approaches.} In recent years, in the clustering task \cite{ran2023comprehensive} (another task as important as the outlier detection task in machine learning), some researchers have abandoned the algorithmic principle optimization strategy (\emph{i.e.}, optimizing the clustering performance by refining the principles of the existing clustering algorithms) \cite{li2023improve, pu2024adaptive, li2021hibog}. They have introduced gravity in physics to force similar objects within the dataset to move closer to each other. This movement renders the distribution of objects more friendly to the clustering task, thereby improving the accuracy of clustering algorithms. Numerous experimental results show that the accuracy of the optimized clustering algorithms often improves by more than 20\% \cite{li2023improve, li2021hibog}. More importantly, since this optimization process is independent of the clustering process, this physics-based optimization strategy is generic and applicable to various clustering principles. Inspired by this, we propose a physics-based \underline{\textbf{O}}ptimization \underline{\textbf{S}}trategy for outlier \underline{\textbf{D}}etection algorithms, called OSD, to address the challenges encountered in outlier detection tasks. Considering the characteristics of outlier detection tasks, OSD no longer introduces gravity from physics, but instead introduces \textbf{‘explosion’} from physics. Specifically, OSD first divides the dataset into several object-blocks (\emph{i.e.}, sets of adjacent objects) based on neighborhood relationships, and assigns a mass value to each object-block according to the number of objects. Although OSD cannot determine which object-blocks contain outliers and which contain normal objects, we have proven through a series of theorems that outliers and normal objects have a high probability of belong to different object-blocks, and that the object-blocks composed of outliers have less mass than the object-blocks composed of normal objects. Next, OSD inserts a virtual bomb in the feature space of the dataset. According to the principles of momentum and impulse in physics \cite{mansfield2020understanding}, after the virtual bomb explodes, object-blocks with small mass will acquire an initial velocity significantly greater than that of object-blocks with large mass, and therefore the outliers rocket away from the normal objects, as shown in Figure \ref{fig:Explosion}. Therefore, in the dataset after the explosion, the outlier detection algorithms can more easily distinguish between outliers and normal objects, leading to a higher accuracy. \emph{Since OSD is independent of the outlier detection process, it is applicable to various outlier detection algorithms with different principles. That is, the proposed optimization strategy, OSD, is generic, addressing \textbf{Challenge 1}. Furthermore, since OSD does not alter the principles of the optimized outlier detection algorithms, OSD can preserve the original advantages of the optimized outlier detection algorithms, addressing \textbf{Challenge 2}.}
\begin{figure}[h]
  \centering
  \includegraphics[width=3in]{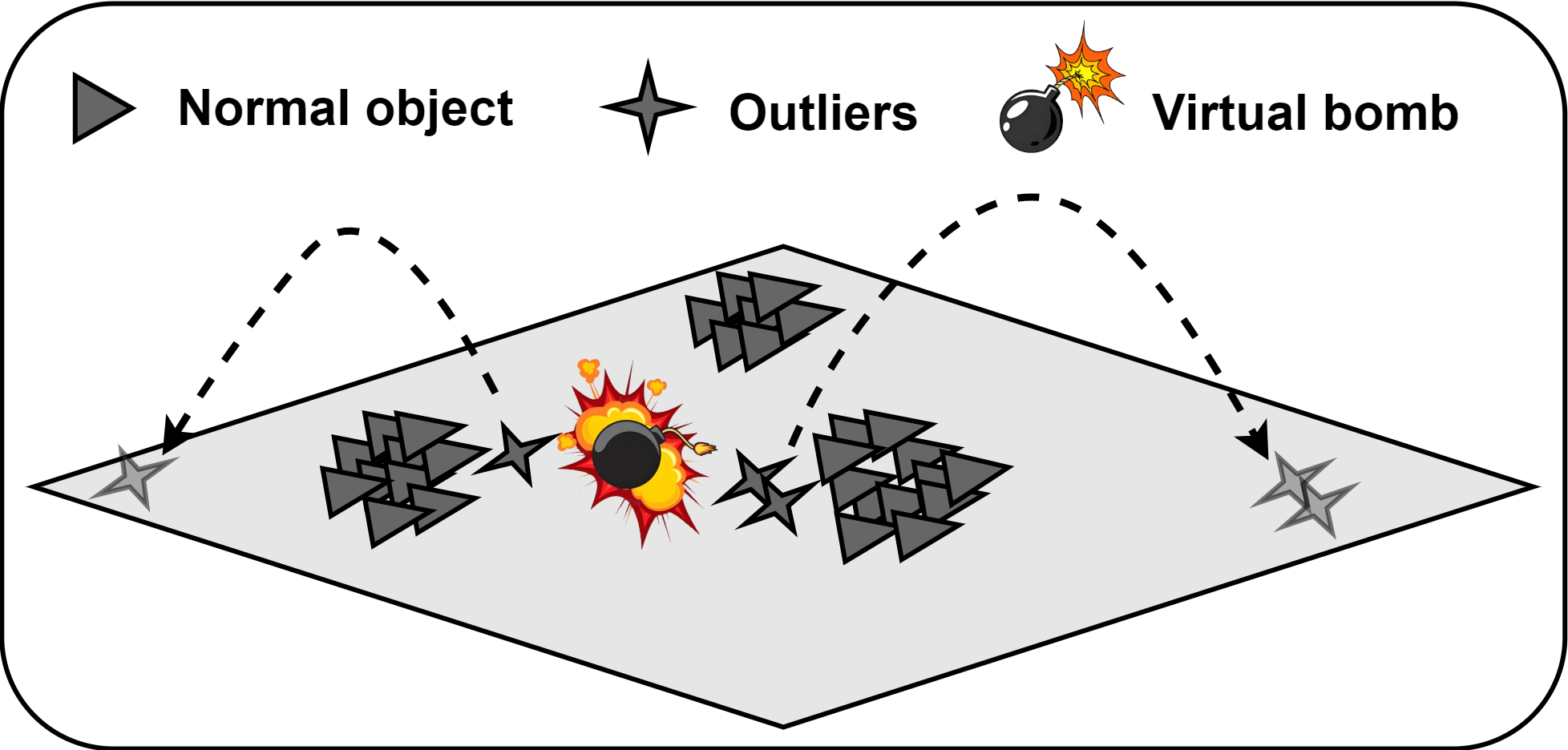}
  \caption{Feature explosion.}
  \label{fig:Explosion}
\end{figure}

\textbf{Main Contributions.} We summaries the main contributions of this paper:

1.) \textbf{We propose the first ‘generic’ optimization strategy for the outlier detection task}, called OSD, which can help different outlier detection algorithms achieve higher accuracy.

2.) \textbf{We are the first to apply physics to the outlier detection task}. By leveraging the principles of momentum and impulse in physics, OSD forces potential outliers to move rapidly away from potential normal objects, thereby reducing the difficulty of detecting outliers for outlier detection algorithms.

3.) \textbf{We are the first to separate the optimization process from the outlier detection process}, enabling the outlier detection algorithm optimized by OSD to preserve its original advantages.

4.) \textbf{Experimental results demonstrate that OSD enhances the accuracy of all optimized outlier detection algorithms}. In terms of average accuracy, these algorithms achieve an improvement of 15\% (AUC) and 63.7\% (AP).

\section{Related works}
\label{sec:Related works}
The proposed OSD aims to improve the accuracy of outlier detection algorithms by introducing physical principles into data analysis, therefore we discuss current related works on outlier detection algorithms and physics-based data analysis methods.

\textbf{Physics-based Data Analysis Methods.} Wright \cite{wright1977gravitational} first introduces gravity from physics into data analysis. He treats each object as a particle, with gravity existing between objects. Under the pull of gravity, all objects move in the feature space. Since then, many researchers begin to explore the application of gravity in the clustering task — a task in data analysis as critical as outlier detection. They usually adjust the distance between objects through gravity, so that objects within the same cluster become closer together, thereby reducing the difficulty of clustering algorithms in identifying clusters. Specifically, Newton \cite{blekas2007newtonian} believes that the dataset follows a Gaussian distribution, so it forces objects to move towards the cluster center in order to make the features of the Gaussian distribution more prominent. Herd \cite{wong2014herd} is similar to Newton, but it focuses more on the magnitude of force and sets a speed limit to avoid objects moving beyond the cluster center. In recent years, instead of forcing objects to be experienced by gravity in a fixed direction, many methods draw on the laws of celestial motion to stipulate that each object is experienced by gravities from multiple surrounding objects. HIBOG \cite{li2021hibog} is one of the most representative methods, and a large number of experiments have confirmed that HIBOG can improve the accuracy of tested clustering algorithms by more than twice. In order to avoid abnormal proximity of adjacent clusters, HIAC \cite{li2023improve} proposed a limited version of the gravity model, which stipulates that gravity only exists between valid neighbors. KDE-AHIAC \cite{pu2024adaptive} further improves HIAC by constructing a decision graph based on kernel density function and introducing an adaptive threshold selection method, making the selection of valid neighbors more convenient. DCLCMS \cite{zhang2022novel} identifies core objects based on the square ratio of gravity to mass, in order to improve the performance of clustering algorithms on datasets with large variations in density and manifold structure. PGCGP \cite{chen2023parallel} converts object movement into grid movement, significantly reducing the complexity of computing gravity on large-scale datasets. HCEG \cite{hao2024hceg} proposes a heterogeneous ensemble clustering method based on gravity to achieve intelligent data pricing. A small number of researchers also attempt to introduce gravity models into the outlier detection task and propose some gravity-based outlier detection algorithms \cite{zhu2022high, xie2020local}. These algorithms do not change the distance between objects just measure the similarity between objects based on gravity. However, they cannot reduce the difficulty of outlier detection algorithms in distinguishing between outliers and normal objects in the same way that the above-mentioned methods reduce the difficulty of clustering algorithms in detecting clusters. \emph{In this paper, for the outlier detection task, we propose an explosion shock force model that forces outliers and normal objects to move away from each other, radically reducing the difficulty of outlier detection algorithms in distinguishing between outliers and normal objects.}

\textbf{Outlier Detection Algorithms.} Outlier detection algorithms can be broadly classified as statistics-based algorithms, density-based algorithms, deep learning-based algorithms, and clustering-based algorithms. Specifically, statistics-based algorithms \cite{aydin2023boundary, li2023ecod, li2020copod, rousseeuw2011robust} usually assume that the dataset follows a certain distribution, and by analyzing the statistical properties of the objects, detect those objects that are significantly different from the overall distribution as outliers. Typically, the principles of statistical-based algorithms are easy to explain and perform well on small and low-dimensional datasets, but perform poorly on datasets that do not conform to known distributions. Density-based algorithms \cite{huang2023novel, zhou2024outlier, aydin2023boundary, breunig2000lof} detect outliers by comparing the local density of an object with its neighbors. Due to the focus on local information, density-based algorithms can identify local outliers, especially on datasets with uneven distribution. Clustering-based algorithms \cite{chen2021block, li2024detecting, rodriguez2014clustering} divide objects into different clusters and then detect those objects that do not belong to any cluster or are at the boundary of a cluster as outliers. Different cluster divisions may lead to vastly different outlier detection results. Deep learning-based algorithms \cite{hojjati2024dasvdd, goodge2022lunar, liu2021rca} have received a lot of attention in recent years, they embed the outlier detection task into neural networks. For example, by calculating the reconstruction error of objects in a self-encoder network, some deep learning-based algorithms detect an object with a large reconstruction error as an outlier. \emph{Obviously, there is an obvious principal barrier between different classes of outlier detection algorithms, which leads to the fact that existing optimization strategies cannot be suitable for different classes of outlier detection algorithms. In this paper, the optimization strategy we propose, OSD, is independent of the outlier detection principles, thus breaking down the barriers between different classes of outlier detection algorithms. As a result, OSD can optimize diverse outlier detection algorithms with vastly different principles.}

\section{The Proposed Method}
\label{sec:OSD}

\subsection{Problem Definition}
For a $d$-dimensional dataset $X$ containing $N$ objects, $X=\left\{x_1,x_2,\cdots,x_N\right\}\subset R^d$, OSD aims to change the position of objects in the feature space, transforming $X$ into $\widehat{X}$ ($\widehat{X}=\left\{\widehat{x_1},\widehat{x_2},\cdots,\widehat{x_N}\right\}\subset R^d$), such that the outliers (see Definition \ref{def:outlier} and Example \ref{exa:outliers and normal objects}) are further away from normal objects (see Definition \ref{def:Normal Object} and Example \ref{exa:outliers and normal objects}) and the distribution of outliers is sparser in $\widehat{X}$ than in $X$. Ultimately, by identify outliers from $\widehat{X}$ instead of $X$, outlier detection algorithms can achieve higher accuracy.

\begin{definition}
\label{def:Normal Object}
\textbf{(Cluster and Normal Object)} For $\mathfrak{X}\subseteq X$, if the objects within $\mathfrak{X}$ are mutual neighbors and the number of objects within $\mathfrak{X}$ is not significantly fewer than the total number of objects in $X$, then $\mathfrak{X}$ is a cluster in $X$. The objects in $\mathfrak{X}$ are normal objects. The $i$-th cluster in $X$ is denoted as ${\mathcal{CLU}}_i$.
\end{definition}

\begin{definition}
\label{def:outlier}
\textbf{(Outlier)} Suppose $X$ contains $f$ clusters, namely ${\mathcal{CLU}}_1,{\mathcal{CLU}}_2,\cdots,{\mathcal{CLU}}_f$. For $\forall x_j\in X$ (i.e., the $j$-th object in $X$), if $x_j\notin\forall{\mathcal{CLU}}_i$, then $x_j$ is an outlier.
\end{definition}

\begin{figure}[h]
  \centering
  \includegraphics[width=3in]{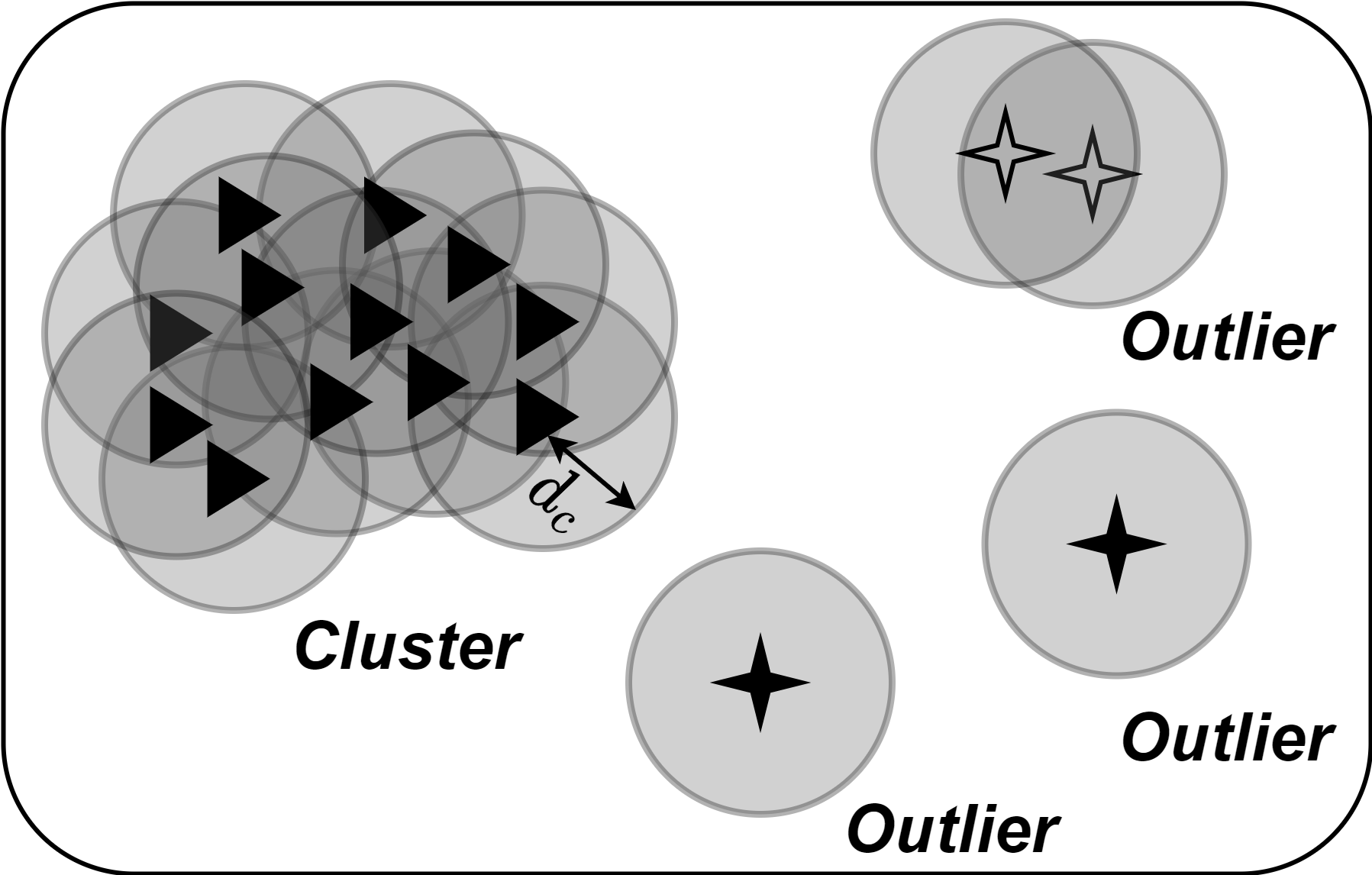}
  \caption{An example about outliers and normal objects.}
  \label{fig:outliers and normal objects}
\end{figure}

\begin{example}
\label{exa:outliers and normal objects}
\textbf{(Clusters, Normal Objects, and Outliers)} Let objects with distance less than $d_c$ be neighbors, where $d_c$ is a small value. In Figure \ref{fig:outliers and normal objects}, the radius of each circular area is $d_c$. Obviously, the objects within the circular area of each object are all its neighbors. By observation, the circular areas of triangular objects overlap with each other, so triangular objects are mutual neighbors. Due to the large number of triangular objects, according to Definition \ref{def:Normal Object}, the set they form is a cluster, and they are normal objects. Although the two hollow star-shaped objects are neighbors to each other, their number is far less than the total number of objects, so the set they form is not a cluster. According to Definition \ref{def:Normal Object} and Definition \ref{def:outlier}, all star-shaped objects are outliers.
\end{example}

\subsection{Overview}
OSD consists of two steps:
\begin{itemize}
\item \textbf{Step 1 (Explosion Process):} OSD inserts a virtual bomb in the feature space of $X$, and then force outliers to rocket away from normal objects through explosion, resulting in transforming $X$ into $\overline{\overline{X}}$, as detailed in Section \ref{sec:Explosion process}.

\item \textbf{Step 2 (Repulsion Process):} After the explosion, to prevent certain outliers from mixing into normal objects, OSD introduces repulsive forces to force non-original neighbors to move away from each other, resulting in transforming $\overline{\overline{X}}$ into $\widehat{X}$, as detailed in Section \ref{sec:Repulsion Process}.
\end{itemize}

\subsection{Explosion Process}
\label{sec:Explosion process}
\textbf{Motivation.} According to Definitions \ref{def:Normal Object} and \ref{def:outlier}, normal objects are those dense objects that are clustered with each other in the feature space, while outliers are those sparse objects scattered in the feature space. All existing outlier detection algorithms essentially identify outliers by distinguishing the differences between outliers and normal objects. Clearly, the greater the difference between outliers and normal objects (\emph{i.e.}, \textbf{the sparser the outliers and the further the outliers are from the normal objects}), the easier it is for outlier detection algorithms to distinguish differences (\emph{i.e.}, \textbf{the greater the probability that the outlier detection algorithms obtain highly accurate results}). Therefore, in this paper, we plan to design a ‘\underline{Feature Explosion}’ mechanism (see Definition \ref{def:Feature Explosion}) to force the outliers to move away from the normal objects and to disperse the outliers.

\begin{definition}
\label{def:Feature Explosion}
\textbf{(Feature Explosion)} The dramatic change in the position of objects in the feature space is called the feature explosion.
\end{definition}

\textbf{Main Idea (Feature Explosion Mechanism).} Inspired by the explosion phenomenon in the real world, OSD inserts a virtual bomb into the feature space of the dataset, and then simulates the explosion to drive outliers away from normal objects, as shown in Figure \ref{fig:Explosion}. Specifically, \textbf{$\bullet$ Step 1 (Section \ref{sec:The object-block division process}):} OSD first divides the dataset into several object-blocks and assigns them different masses. Although OSD cannot determine which object-blocks contain outliers and which contain normal objects, we have proven through a series of theorems that outliers and normal objects have a high probability of belong to different object-blocks (see Remark \ref{rem:First Characteristic} for details), and that the object-blocks composed of outliers have less mass than the object-blocks composed of normal objects (see Remark \ref{rem:Second Characteristic} for details). \textbf{$\bullet$ Step 2 (Section \ref{sec:The explosion process}):} OSD detonates the virtual bomb. According to the principles of momentum and impulse in physics \cite{mansfield2020understanding}, object-blocks with different masses will acquire different initial velocities during the explosion. Therefore, OSD can control the movement of object-blocks based on different initial velocities, such that object-blocks with small masses are rocket away from those with large masses. As a result, the outliers are rocket away from the normal objects. Below, we will describe the object-block division process (Section \ref{sec:The object-block division process}) and the explosion process (Section \ref{sec:The explosion process}) in detail.

\subsubsection{The Object-block Division Process}
\label{sec:The object-block division process}

\begin{definition}
\label{def:kNN neighbors}
\textbf{($k$NN neighbors)} For $\forall x_i\in X$ and a positive integer $k$, the $k$NN neighbors of $x_i$, denoted as $ N_k(x_i)$, is a subset of $X$, satisfying the following conditions: 1.) $N_k(x_i)$ contains $k$ objects $x_{i_1},x_{i_2},\cdots,x_{i_k}$, in which $i_1,i_2,\cdots i_k\in\{1,2,\cdots N\}$; 2.) For $\forall x_j\in X-N_k(x_i)$ and $\forall x_g\in N_k(x_i)$, $\|x_g-x_i\|_2\le \|x_j-x_i\|_2$, in which $\|x_g-x_i\|_2$ is the Euclidean Distance between $x_g$ and $x_i$.
\end{definition}

\begin{example}
\label{exa:kNN neighbors}
\textbf{($k$NN neighbors)} For $X\subset R^3$, $X=\left\{x_1,x_2,x_3,x_4\right\}$, where $x_1=\langle 1,0,0\rangle$, $x_2=\langle 2,0,0\rangle$, $x_3=\langle 3,0,0\rangle$, and $x_4=\langle 4,0,0\rangle$. $\|x_1-x_2\|_2=\sqrt{(1-2)^2+(0-0)^2+(0-0)^2}=1$. Similarly, $\|x_1-x_3\|_2=2$, $\|x_1-x_4\|_2=3$. If $k=2$, then the $k$NN neighbors of $x_1$ are $x_2$ and $x_3$.
\end{example}

\begin{figure}[h]
  \centering
  \includegraphics[width=3in]{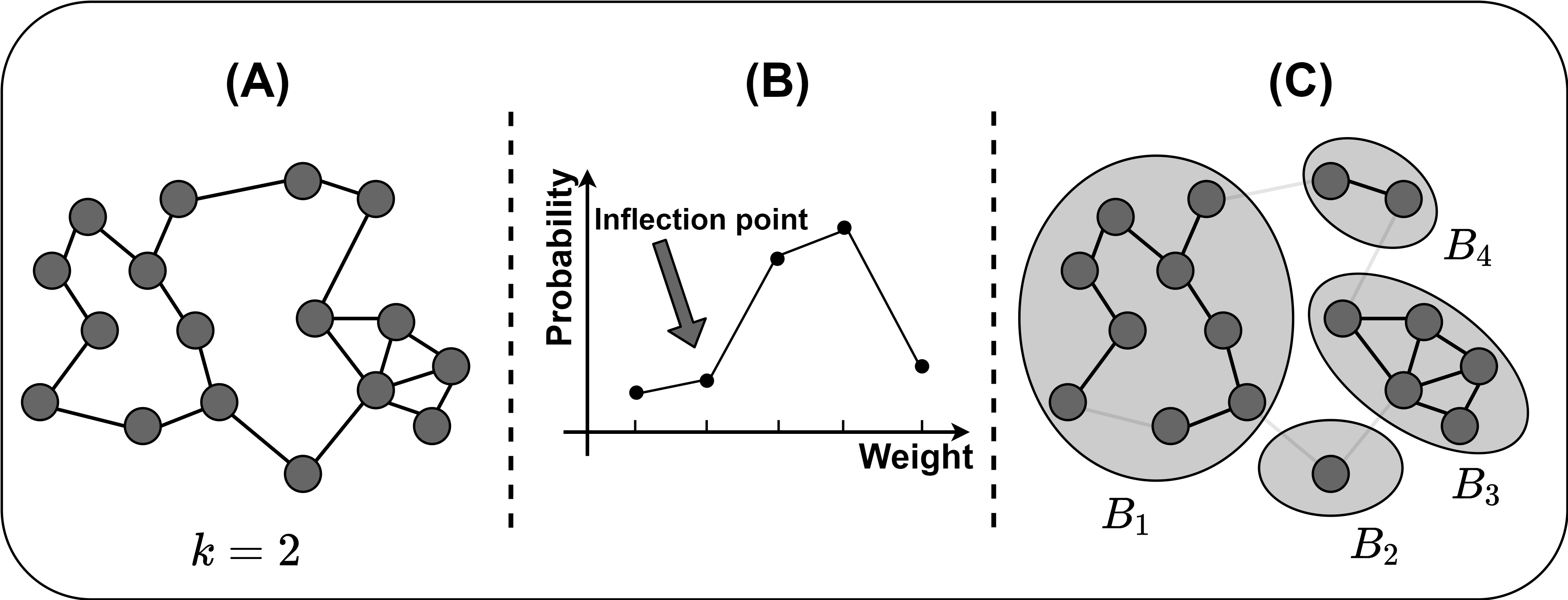}
  \caption{The object-block division.}
  \label{fig:object-block division}
\end{figure}

OSD generates a $k$-nearest neighbor graph for dataset $X$, where each object is connected to its $k$NN neighbors by edges, as shown in Figure \ref{fig:object-block division}(A). Specifically, if there is an edge between object $x_i$ (\emph{i.e.}, the $i$-th object in $X$) and object $x_j$, and then the edge is denoted as $e_{ij}$, and its weight is defined as
\begin{equation}
\omega\left(e_{ij}\right)=-\|x_i-x_j\|_2.
\label{eq:weight}
\end{equation}
The farther the distance between object $x_i$ and object $x_j$, the smaller the weight of $e_{ij}$. After counting the weight values of all edges, OSD computes the probability distribution of these weight values. Specifically, OSD first divides the range of these weight values, $\left[\mathop{\min}\limits_{i,j\le N}{\left(\omega\left(e_{ij}\right)\right)},\mathop{\max}\limits_{i,j\le N}{\left(\omega\left(e_{ij}\right)\right)}\right]$, into several equidistant intervals, each with a length of $\frac{\left(\mathop{\max}\limits_{i,j\le N}{\left(\omega\left(e_{ij}\right)\right)}-\mathop{\min}\limits_{i,j\le N}{\left(\omega\left(e_{ij}\right)\right)}\right)\cdot 10}{N}$. For the $g$-th interval $\Delta_g$, its probability value is
\begin{equation}
\mathcal{P}(\Delta_g)=\frac{\sum_{i,j\le N}\varphi(\omega\left(e_{ij}\right)|\Delta_g)}{N},
\label{eq:probability}
\end{equation}
in which, if $\omega\left(e_{ij}\right)\in\Delta_g$, then $\varphi\left(\omega\left(e_{ij}\right)\middle|\Delta_g\right)=1$; Otherwise, $\varphi\left(\omega\left(e_{ij}\right)\middle|\Delta_g\right)=0$. For the $k$-nearest neighbor graph in Figure \ref{fig:object-block division}(A), the probability distribution curve of weight values is shown in Figure \ref{fig:object-block division}(B).

Due to the fact that each object in the $k$-nearest neighbor graph is only connected to its $k$NN neighbors, the number of the edges with large weight values is significantly higher than that of the edges with small weight values. Therefore, the probability distribution curve inevitably has a clear inflection point, as indicated by the arrow in Figure \ref{fig:object-block division}(B). OSD treats the inflection point as a threshold and clips edges with weight values less than the threshold  (we will discuss the impact of this threshold on the results in the Section \ref{sec:Robustness experiments}). Finally, in the pruned $k$-nearest neighbor graph, the set of objects within each connected subgraph is an object-block, as defined in Definition \ref{def:Object-block}. The number of objects within an object-block is referred to as the mass of the object-block, as defined in Definition \ref{def:mass}.

\begin{definition}
\label{def:Object-block}
\textbf{(Object-block)} In the pruned $k$-nearest neighbor graph, for $\forall x_i,x_j\in\mathcal{A}\subseteq X$, if $\exists\left\{a_1,a_2,\cdots,a_Z\right\}\subseteq\mathcal{A}$ such that $x_i$ is connected to $a_1$, $x_j$ is connected to $a_Z$, and $a_l$ is connected to $a_{l+1}$ (for $\forall l<Z$); in addition, for $\forall x_g\in X-\mathcal{A}$, if $\nexists x_t\in \mathcal{A}$ is connected to $x_g$, then $\mathcal{A}$ is an object-block. The $i$-th object-block in $X$ is denoted as $\mathcal{B}_i$.
\end{definition}

\begin{definition}
\label{def:mass}
\textbf{(Mass)} For $\forall\mathcal{B}_i\subset X$, if $\mathcal{B}_i$ contains $m$ objects, then the mass of $\mathcal{B}_i$ is $m$, denoted as $\mathcal{M}_i=m$.
\end{definition}

\begin{example}
\label{exa:Object-block and mass}
\textbf{(Object-block and Mass)} For the $k$-nearest neighbor graph in Figure \ref{fig:object-block division}(A), when the value indicated by the arrow in Figure \ref{fig:object-block division}(B) is set as the threshold and edges with weights less than this threshold are clipped, $X$ is divided into 4 object-blocks (see Figure \ref{fig:object-block division}(C)). The first object-block $\mathcal{B}_1$ contains 9 objects, so $\mathcal{M}_1=9$. Similarly, $\mathcal{M}_2=1$, $\mathcal{M}_3=5$, and $\mathcal{M}_4=2$.
\end{example}

We prove through a series of theorems that the object-blocks have two characteristics (see Remark \ref{rem:First Characteristic} and Remark \ref{rem:Second Characteristic} for details), which will be crucial for controlling the explosion process in Section \ref{sec:The explosion process}.

\begin{remark}
\label{rem:First Characteristic}
\textbf{(First Characteristic)} Outliers and normal objects are highly likely to belong to different object-blocks, as proven in Theorem \ref{the:diffrent block}.
\end{remark}

\begin{remark}
\label{rem:Second Characteristic}
\textbf{(Second Characteristic)} The mass of the object-block composed of outliers is always smaller than the mass of the object-block composed of normal objects, as proven in Theorem \ref{the:mass}.
\end{remark}

\begin{lemma}
\label{lemma}
Let $\mathcal{E}$ be the set of remaining edges in the pruned $k$-nearest neighbor graph. For $\forall e\in\mathcal{E}$, $\mathcal{P}_{normal}(e)\gg\mathcal{P}_{outlier}(e)$, in which $\mathcal{P}_{outlier}(e)$ is the probability that $e$ is an outlier-edge (i.e., an edge connecting at least one outlier), and $\mathcal{P}_{normal}(e)$ is the probability that $e$ is a normal-edge (i.e., an edge connecting only normal objects). In other words, the outlier-edge is a very low probability event in $\mathcal{E}$.
\end{lemma}
\begin{proof}
OSD is based on a fundamental assumption that the number of outliers in the dataset is much smaller than the number of normal objects. This assumption aligns with objective laws in the real world, and nearly all outlier detection algorithms are based on this assumption \cite{panjei2022survey}. Therefore, the edges connecting outliers are few. In addition, according to Definition \ref{def:Normal Object}, normal objects are close to each other, so the $k$NN neighbors of normal objects are almost also normal objects. That is, for an edge whose one endpoint is a normal object, its another endpoint is always a normal object. And according to Definition \ref{def:outlier}, outliers are sparsely distributed, so the edges connecting outliers are necessarily long. Since OSD prunes the $k$-nearest neighbor graph by clipping long edges (\emph{i.e.}, the edges with small weights), the number of edges connecting outliers is further reduced. In conclusion, the outlier-edge is a very low probability event in $\mathcal{E}$.
\end{proof}

\begin{theorem}
\label{the:diffrent block}
Let $X$ be divided into $C$ object-blocks, $\mathcal{B}_1,\mathcal{B}_2,\cdots,\mathcal{B}_C$. For $\forall x_i,x_j\in X$, if $x_i$ is a normal object and $x_j$ is an outlier, then with high probability, $\nexists\mathcal{B}_y\in\left\{\mathcal{B}_1,\mathcal{B}_2,\cdots,\mathcal{B}_C\right\}$ such that $x_i,x_j\in\mathcal{B}_y$.
\end{theorem}
\begin{proof}
Assume that $\exists\mathcal{B}_y\in\left\{\mathcal{B}_1,\mathcal{B}_2,\cdots,\mathcal{B}_C\right\}$ such that $ x_i,x_j\in\mathcal{B}_y$. That is, after clipping the edges, there exists an edge between an outlier and a normal object in $\mathcal{B}_y$. Let the weight of this edge fall within the $p$-th interval of weight values, and let the inflection point of the probability distribution curve lie within the $q$-th interval.\\
$\because$ The edge between an outlier and a normal object in $\mathcal{B}_y$ is not clipped.\\
$\therefore$ $p>q$.\\
$\because$ According to Definition \ref{def:Normal Object} and Definition \ref{def:outlier}, outliers are sparsely distributed, while normal objects are close to each other.\\
$\therefore$ The weight values of the edges between outliers and normal objects are always smaller than the weight values of the edges between normal objects.\\
$\therefore$ With high probability, the edges with weight values in between the $q$-th interval and the $p$-th interval are the edges connecting to outliers, \emph{i.e.}, outlier-edges.\\
$\because$ In the probability distribution curve, the probability on the left side of the inflection point (\emph{i.e.}, the $q$-th interval) is extremely small, while the probability increases sharply on the right side of the inflection point. Therefore, the inflection point is the boundary between the high probability event and the low probability event.\\
$\therefore$ The edge with weight value in between the $q$-th interval and the $p$-th interval belongs to a high probability event.\\
$\therefore$ With high probability, the outlier-edge is a high probability event.\\
$\therefore$ With high probability, the assumption contradicts Lemma \ref{lemma}, so Theorem \ref{the:diffrent block} is proved.
\end{proof}

\begin{theorem}
\label{the:mass}
For $\forall\mathcal{B}_i,\mathcal{B}_j\subset X$, if $\mathcal{B}_i$ is an object-block composed of outliers and $\mathcal{B}_j$ is an object-block composed of normal objects, then $\mathcal{M}_i<\mathcal{M}_j$.
\end{theorem}
\begin{proof}
Object-blocks are divided by clipping long edges, so objects which are far apart are split into different object-blocks. According to Definition \ref{def:Normal Object} and \ref{def:outlier}, normal objects have large-scale aggregation (\emph{i.e.}, a large number of normal objects are close to each other), while outliers do not have such aggregation. Therefore, outliers are split more severely than normal objects.
\end{proof}

We describe the detailed implementation of object-block division in Algorithm \ref{alg:object-block division}.

\begin{algorithm}
\caption{The Object-block Division}
\label{alg:object-block division}
\KwIn{$X$, $k$}
\KwOut{$\{\mathcal{B}_1,\mathcal{B}_2,\cdots,\mathcal{B}_C\}$, $\{\mathcal{M}_1,\mathcal{M}_2,\cdots,\mathcal{M}_C\}$}
\ Calculating the weight for each edge according to the formula (\ref{eq:weight}).\\
\ Calculating the probability for weight values according to the formula (\ref{eq:probability}).\\
\ Generating a probability distribution curve and treating the inflection point as the threshold.\\
\ Clipping edges with weight values less than the threshold.\\
\ According to Definition \ref{def:Object-block}, searching for all connected subgraphs in the pruned $k$-nearest neighbor graph to obtain $\mathcal{B}_1,\mathcal{B}_2,\cdots,\mathcal{B}_C$ and $\mathcal{M}_1,\mathcal{M}_2,\cdots,\mathcal{M}_C$.\\
\Return{$\{\mathcal{B}_1,\mathcal{B}_2,\cdots,\mathcal{B}_C\}$, $\{\mathcal{M}_1,\mathcal{M}_2,\cdots,\mathcal{M}_C\}$}
\end{algorithm}

\subsubsection{The Explosion Process}
\label{sec:The explosion process}
We insert a virtual bomb into the feature space of dataset $X$, and construct a physical motion model for object-blocks under the explosion. This model is subdivided into two parts: \textbf{Virtual Bomb Model} and \textbf{Displacement Model}. To avoid complex analysis, we treat each object-block as a particle, as defined in Definition \ref{def:particle}.

\begin{definition}
\label{def:particle}
\textbf{(Particle and Mass)} For $\forall\mathcal{B}_i\subset X$, let $\mathcal{B}_i=\left\{x_{i_1},x_{i_2},\cdots,x_{i_Z}\right\}$, in which $i_1,i_2,\cdots,i_Z\in\{1,2,\cdots,N\}$. The particle of $\mathcal{B}_i$ is denoted as $\mathcal{B}_i^\dag$, $\mathcal{B}_i^\dag=\frac{\sum_{z=1}^{Z}x_{i_z}}{Z}$. The mass of $\mathcal{B}_i^\dag$ is the same as that of $\mathcal{B}_i$, and is still denoted as $\mathcal{M}_i$. In other words, $\mathcal{B}_i^\dag$ is essentially a point with mass, and it serves as a substitute for $\mathcal{B}_i$.
\end{definition}

\begin{example}
\label{exa:particle}
\textbf{(Particle and Mass)} For $X\subset R^2$, if its first object-block $\mathcal{B}_1=\{\langle 1,2\rangle,\langle 1,3\rangle\}$, then $\mathcal{B}_1^\dag=\langle\frac{1+1}{2},\frac{2+3}{2}\rangle=\langle 1,2.5\rangle$. Since $\mathcal{B}_1$ contains 2 objects, the mass of $\mathcal{B}_1^\dag$ is 2, i.e., $\mathcal{M}_1=2$.
\end{example}

\textbf{Virtual Bomb Model.} The virtual bomb is denoted as $\Theta$. In Definition \ref{def:shock force}, we define the explosion shock force exerted on each particle during the explosion of the virtual bomb $\Theta$, as detailed in Examples \ref{exa:Constant G} and \ref{exa:Virtual bomb and explosion shock force}. In Theorem \ref{the:bomb location}, we prove that the closer the virtual bomb $\Theta$ is to the centroid of the particles, the smaller the sum of squared distances between the particles and the virtual bomb $\Theta$. According to Definition \ref{def:shock force}, the explosion shock force is inversely proportional to the squared distance between the particle and  the virtual bomb $\Theta$. Therefore, the closer the virtual bomb $\Theta$ is to the centroid of the particles, the greater the total impact of the explosion on the dataset. To achieve the optimal explosive effect, we place the virtual bomb $\Theta$ at the centroid of the particles, \emph{i.e.}, 

\begin{equation}
\Theta=\frac{\sum_{i=1}^{C}\mathcal{B}_i^\dag}{C},
\label{eq:virtual bomb}
\end{equation}
in which $C$ is the number of object-blocks in $X$.

\begin{definition}
\label{def:shock force}
\textbf{(Explosion Shock Force)} When the virtual bomb $\Theta$ explodes, for $\forall\mathcal{B}_i^\dag$, the explosion shock force exerted on $\mathcal{B}_i^\dag$ is denoted as $\mathcal{F}_i$, $\mathcal{F}_i=G\cdot\frac{1}{\|\mathcal{B}_i^\dag-\Theta\|_2}\cdot\frac{\mathcal{B}_i^\dag-\Theta}{\|\mathcal{B}_i^\dag-\Theta\|_2}$. Specifically, $G$ is a constant that determines order of magnitude of the explosion shock force, $G=\frac{1}{N}\sum_{j=1}^{N}{\|x_j-x_{j|k}\|_2}$, in which $x_{j|k}$ is the $k$-th nearest neighbor of $x_j$, and $k$ is the input parameter in Section \ref{sec:The object-block division process}; $\frac{1}{\|\mathcal{B}_i^\dag-\Theta\|_2}$ controls the scale of the explosion shock force, the closer $\mathcal{B}_i^\dag$ is to $\Theta$, the larger the explosion shock force exerted on $\mathcal{B}_i^\dag$; $\frac{\mathcal{B}_i^\dag-\Theta}{\|\mathcal{B}_i^\dag-\Theta\|_2}$ determines the direction of the explosion shock force, pointing from $\Theta$ to $\mathcal{B}_i^\dag$.
\end{definition}

\begin{example}
\label{exa:Constant G}
\textbf{(Constant G)} For $X\subset R^3$, $X=\left\{x_1,x_2,x_3,x_4\right\}$, where $x_1=\langle 1,0,0\rangle$, $x_2=\langle 2,0,0\rangle$, $x_3=\langle 3,0,0\rangle$, and $x_4=\langle 4,0,0\rangle$. If $k=2$, then $x_{1|k}=x_3$, $x_{2|k}=x_4$, $x_{3|k}=x_1$, $x_{4|k}=x_2$. Thus, according to Definition \ref{def:shock force}, $G=\frac{1}{4}\sum_{j=1}^{4}{\|x_j-x_{j|k}\|_2}=\frac{\sqrt{2}}{4}$.
\end{example}

\begin{figure}[h]
  \centering
  \includegraphics[width=3in]{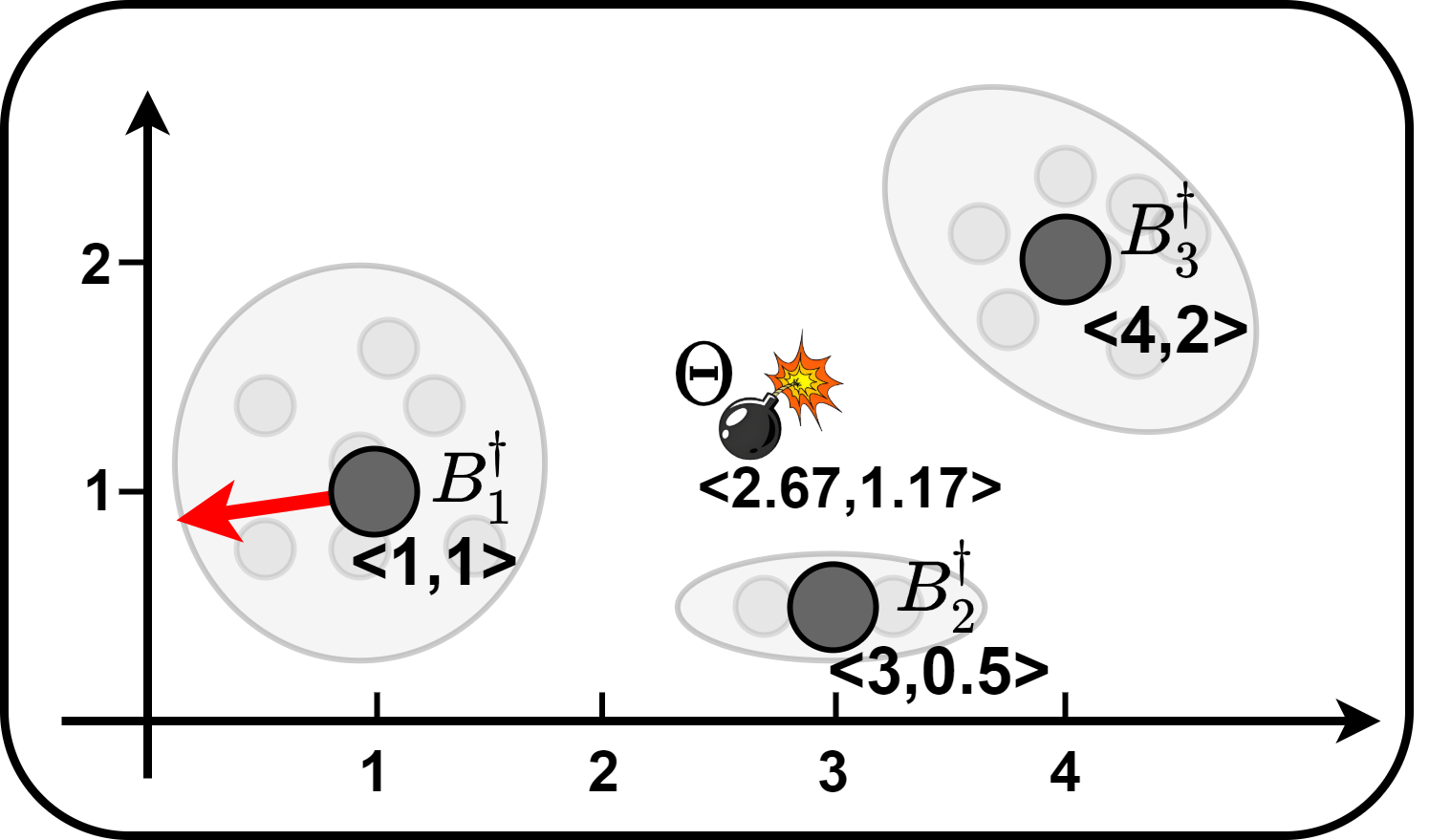}
  \caption{An example of virtual bomb and explosion shock force.}
  \label{fig:shock_force}
\end{figure}

\begin{example}
\label{exa:Virtual bomb and explosion shock force}
\textbf{(Virtual Bomb and Explosion Shock Force)} For $X\subset R^2$ with three object-blocks, the particles of object-blocks are $\mathcal{B}_1^\dag=\langle 1,1\rangle$, $\mathcal{B}_2^\dag=\langle 3,0.5\rangle$, $\mathcal{B}_3^\dag=\langle 4,2\rangle$, as shown in Figure \ref{fig:shock_force}. According to the formula (\ref{eq:virtual bomb}), the virtual bomb $\Theta=\frac{\langle 1,1\rangle +\langle3,0.5\rangle +\langle 4,2\rangle}{3}=\langle \frac{8}{3}, \frac{3.5}{3}\rangle=\langle 2.67,1.17\rangle$. Let $G=5$, according to Definition \ref{def:shock force}, the explosion shock force exerted on $\mathcal{B}_1^\dag$ is $\mathcal{F}_1=5\cdot\frac{1}{\| \langle 1,1\rangle-\langle 2.67,1.17\rangle\|_2}\cdot\frac{\langle 1,1\rangle-\langle 2.67,1.17\rangle}{\| \langle 1,1\rangle-\langle 2.67,1.17\rangle\|_2}=\langle -2.97,-0.3\rangle$. The direction of $\mathcal{F}_1$ is illustrated by the red arrow in Figure \ref{fig:shock_force}.
\end{example}

\begin{theorem}
\label{the:bomb location}
If $\|\Theta^\star -\frac{\sum_{i=1}^C{\mathcal{B}_i^\dag}}{C}\|_2<\|\Theta^\diamond -\frac{\sum_{i=1}^C{\mathcal{B}_i^\dag}}{C}\|_2$, then $\sum_{i=1}^C{\|\mathcal{B}_i^\dag -\Theta^\star\|_2^2}<\sum_{i=1}^C{\|\mathcal{B}_i^\dag -\Theta^\diamond\|_2^2}$.
\end{theorem}
\begin{proof}
Let $g(y)=\sum_{i=1}^{C}\|\mathcal{B}_{i}^{\dag}-y\|_{2}^{2}$. Let $X\subset R^d$, so $y=\langle \boxed{y|1}, \boxed{y|2}, \cdots, \boxed{y|d}\rangle$, in which $\boxed{y|2}$ is the value on the second dimensional feature of $y$. Similarly, $\mathcal{B}_{i}^{\dag}=\langle\boxed{\mathcal{B}_{i}^{\dag}|1}, \boxed{\mathcal{B}_{i}^{\dag}|2}, \cdots, \boxed{\mathcal{B}_{i}^{\dag}|d}\rangle$. Obviously, we can rewrite $g(y)$ as $g(\boxed{y|1}, \boxed{y|2}, \cdots, \boxed{y|d})=\sum_{i=1}^{C} \sum_{j=1}^{d}(\boxed{\mathcal{B}_{i}^{\dag}|j}-\boxed{y|j})^{2}$. Through differentiation, for $\forall j\leq d$, $\frac{\partial g}{\partial \boxed{y|j}}=2\sum_{i=1}^{C}(\boxed{y|j}-\boxed{\mathcal{B}_{i}^{\dag}|j})$. If $\left(\begin{matrix}\frac{\partial g}{\partial\fbox{$y\left|1\right.$}}\\\vdots\\\frac{\partial g}{\partial\fbox{$y\left|d\right.$}}\\\end{matrix}\right)=0$, then $\left(\begin{matrix}\fbox{$y\left|1\right.$}\\\vdots\\\fbox{$y\left|d\right.$}\\\end{matrix}\right)=\left(\begin{matrix}\frac{\sum_{i=1}^{C}\fbox{$\mathcal{B}_i^\dag\left|1\right.$}}{C}\\\vdots\\\frac{\sum_{i=1}^{C}\fbox{$\mathcal{B}_i^\dag\left|d\right.$}}{C}\\\end{matrix}\right)$. Hessian Matrix of $g(y)$ is $H\left(g\right)=\left[\begin{matrix}\frac{\partial^2g}{\partial{\fbox{$y\left|1\right.$}}^2}&\ldots&\frac{\partial^2g}{\partial\fbox{$y\left|1\right.$}\fbox{$y\left|d\right.$}}\\\vdots&\ddots&\vdots\\\frac{\partial^2g}{\partial\fbox{$y\left|d\right.$}\fbox{$y\left|1\right.$}}&\ldots&\frac{\partial^2g}{\partial{\fbox{$y\left|d\right.$}}^2}\\\end{matrix}\right]=\left[\begin{matrix}\begin{matrix}2C&0\\0&2C\\\end{matrix}&\begin{matrix}\ldots&0\\\ldots&0\\\end{matrix}\\\begin{matrix}\vdots&\vdots\\0&0\\\end{matrix}&\begin{matrix}\ddots&\vdots\\\ldots&2C\\\end{matrix}\\\end{matrix}\right]$. We can derive the following:\\
$\because C>0$, \emph{i.e.}, the leading principal minors of $H(g)$ are all greater than 0.\\
$\therefore$ When $\left(\begin{matrix}\fbox{$y\left|1\right.$}\\\vdots\\\fbox{$y\left|d\right.$}\\\end{matrix}\right)=\left(\begin{matrix}\frac{\sum_{i=1}^{C}\fbox{$\mathcal{B}_i^\dag\left|1\right.$}}{C}\\\vdots\\\frac{\sum_{i=1}^{C}\fbox{$\mathcal{B}_i^\dag\left|d\right.$}}{C}\\\end{matrix}\right)=\frac{\sum_{i=1}^{C}\mathcal{B}_i^\dag}{C}$, $H(g)$ is positive definite.\\
$\therefore$ $g(y)$ achieves its minimum at $\frac{\sum_{i=1}^{C} \mathcal{B}_{i}^{\dag}}{C}$.\\
$\therefore$ $\sum_{i=1}^C{\|\mathcal{B}_i^\dag -\Theta^\star\|_2^2}<\sum_{i=1}^C{\|\mathcal{B}_i^\dag -\Theta^\diamond\|_2^2}$.
\end{proof}

\textbf{Displacement Model.} We use $\mathcal{B}_i^\dag$ as an example to illustrate its movement process under the explosion shock force. According to physical principles \cite{mansfield2020understanding}, during the explosion, the impulse $I_i$ ($I_i=\mathcal{F}_i\cdot T$, where $T$ is the explosion duration) gained by $\mathcal{B}_i^\dag$ is entirely converted into the momentum $P_i$ ($P_i=\mathcal{M}_i\cdot V_i$, where $V_i$ is the velocity). That is, $\mathcal{F}_i\cdot T=\mathcal{M}_i\cdot V_i$. Therefore, the explosion imparts an initial velocity $V_i=\frac{\mathcal{F}_i\cdot T}{\mathcal{M}_i}$ to $\mathcal{B}_i^\dag$. To simulate the explosion scenario in the real world, we assume that $\mathcal{B}_i^\dag$ is subject to a constant friction force (denoted as $f_i$), with a friction coefficient $\mu$. Clearly, under the influence of the friction force, $\mathcal{B}_i^\dag$ will undergo uniform deceleration motion until its velocity becomes 0. According to physical principles, $f_i=\mathcal{M}_i\cdot\mu$, so the acceleration of $\mathcal{B}_i^\dag$ in uniform deceleration motion is $a=\frac{f_i}{\mathcal{M}_i}=\frac{\mathcal{M}_i\cdot\mu}{\mathcal{M}_i}=\mu$. Therefore, it can be inferred that the duration of uniform deceleration motion of $\mathcal{B}_i^\dag$ is $t_i=\frac{V_i}{a}=\frac{V_i}{\mu}=\frac{\mathcal{F}_i\cdot T}{\mathcal{M}_i\cdot\mu}$. According to the displacement formula of uniform deceleration motion \cite{mansfield2020understanding}, the displacement of $\mathcal{B}_i^\dag$ under the explosion shock force is $S_i=V_i\cdot t_i-\frac{1}{2}a\cdot{t_i}^2=\frac{{\mathcal{F}_i}^2\cdot T^2}{2\mu{\cdot\mathcal{M}_i}^2}$. The feature explosion is not a real physical process, and $\mu$ is virtual, so we set $\mu$ to 0.5 to eliminate the coefficient in the denominator. Ultimately, the displacement of $\mathcal{B}_i^\dag$ is
\begin{equation}
S_i=\frac{{\mathcal{F}_i}^2\cdot T^2}{{\mathcal{M}_i}^2}.
\label{eq:displacement}
\end{equation}
Since $\mathcal{B}_i^\dag$ serves as a substitute for $\mathcal{B}_i$, the displacement of each object in $\mathcal{B}_i$ is also $S_i$. Therefore, for $\forall x_j\in\mathcal{B}_i$, after the explosion, $x_j$ will transform into $\overline{\overline{x_j}}$, 
\begin{equation}
\overline{\overline{x_j}}=x_j+S_i=x_j+\frac{{\mathcal{F}_i}^2\cdot T^2}{{\mathcal{M}_i}^2}.
\label{eq:object after explosion}
\end{equation}
We denote the dataset after the explosion as $\overline{\overline{X}}$. Finally, according to the formula (\ref{eq:object after explosion}), we update $\mathcal{B}_i$ to $\overline{\overline{\mathcal{B}_i}}$, and recalculate the particle of $\overline{\overline{\mathcal{B}_i}}$ with reference to Definition \ref{def:particle}, which is denoted as ${\overline{\overline{\mathcal{B}_i}}}^\dag$. The mass of $\overline{\overline{\mathcal{B}_i}}$ (and ${\overline{\overline{\mathcal{B}_i}}}^\dag$) is the same as that of $\mathcal{B}_i$ (and $\mathcal{B}_i^\dag$), that is $\overline{\overline{\mathcal{M}_i}}=\mathcal{M}_i$. Example \ref{exa:Feature explosion} shows an example of feature explosion. 

We prove through Theorem \ref{the:far} and Theorem \ref{the:sparse} that, compared with $X$, $\overline{\overline{X}}$ has the following advantages: 
\begin{remark}
\label{rem:Far distance}
\textbf{(Far Distance)} Theorem \ref{the:far} proves that in $X$, if the distance between an outlier $x_i$ and the virtual bomb $\Theta$ is close to the distance between a normal object $x_j$ and the virtual bomb $\Theta$, then in $\overline{\overline{X}}$, $x_i$ will be farther away from $\Theta$ than  $x_j$. Therefore, in $\overline{\overline{X}}$, most of outliers will be far away from normal objects, which makes it easier for the outlier detection algorithms to detect outliers.
\end{remark}

\begin{remark}
\label{rem:Sparser}
\textbf{(Sparser)} As is well known, sparsity is an important criterion that distinguishes outliers from normal objects \cite{panjei2022survey}. Theorem \ref{the:sparse} proves that the blocks-objects that are close to each other in $X$ will become far away from each other in $\overline{\overline{X}}$. Therefore, after the explosion, the distribution of outliers will become sparser (Note: normal objects are still dense because they are embedded in dense object-blocks, see Theorem \ref{the:mass} for details). Hence, the feature explosion is beneficial to detecting outliers.
\end{remark}

\begin{example}
\label{exa:Feature explosion}
\textbf{(Feature Explosion)} For $X \subset R^{3}$, in which $\mathcal{B}_{4}=\left\{x_{3}, x_{7}, x_{9}\right\}, x_{3}=\langle 1,4,2\rangle, x_{7}=\langle 0,3,5\rangle, x_{3}=\langle-1,7,2\rangle$. Let the explosion shock force $\mathcal{F}_{4}$ exerted on $\mathcal{B}_{4}^{\dagger}$ be $\langle 3,2,1\rangle$, and let $T=1$. Then after the explosion, the displacement $S_{4}$ of $\mathcal{B}_{4}^{\dagger}$ is $\frac{\langle 3,2,1\rangle^{2} \cdot 1^{2}}{3^{2}}=\left\langle 1, \frac{4}{9}, \frac{1}{9}\right\rangle$. As a result, $\overline{\overline{x_{3}}}=\langle 1,4,2\rangle+$ $\left\langle 1, \frac{4}{9}, \frac{1}{9}\right\rangle=\left\langle 2, \frac{40}{9}, \frac{19}{9}\right\rangle$. Similarly, $\overline{\overline{x_{7}}}=\left\langle 1, \frac{31}{9}, \frac{46}{9}\right\rangle, \overline{\overline{x_{9}}}=\left\langle 0, \frac{67}{9}, \frac{19}{9}\right\rangle$. According to Definition \ref{def:particle}, $\overline{\overline{\mathcal{B}_4}}^{\dagger}=\left\langle\frac{2+1+0}{3}, \frac{\frac{40}{9}+\frac{31}{9}+\frac{67}{9}}{3}, \frac{\frac{19}{9}+\frac{46}{9}+\frac{19}{9}}{3}\right\rangle=\left\langle 1, \frac{138}{27}, \frac{84}{27}\right\rangle, \overline{\overline{\mathcal{M}_{4}}}=3$.
\end{example}

\begin{theorem}
\label{the:far}
For $\forall x_i, x_j\in X$, if $x_i$ is an outlier and $x_j$ is a normal object, and $\left\|x_i-\Theta\right\|_2 \approx$ $\left\|x_j-\Theta\right\|_2$, then $\left\|\overline{\overline{x_i}}-\Theta\right\|_2>\left\|\overline{\overline{x_j}}-\Theta\right\|_2$.
\end{theorem}
\begin{proof}
Let $x_{i}\in\mathcal{B}_{\star}$ and $x_{j}\in \mathcal{B}_\diamond$.\\
$\therefore$ According to the formula (\ref{eq:object after explosion}), $\overline{\overline{x_i}}=x_{i}+S_{\star}, \overline{\overline{x_j}}=x_{j}+S_{\diamond}$.\\
$\because\left\|x_{i}-\Theta\right\|_{2} \approx\left\|x_{j}-\Theta\right\|_{2}$.\\
$\therefore$ To prove $\left\|\overline{\overline{x_i}}-\Theta\right\|_{2}>\left\|\overline{\overline{x_{j}}}-\Theta\right\|_{2}$, it suffices to prove that $\left|S_{\star}\right|>\left|S_{\diamond}\right|$.\\
$\because\left\|x_{i}-\Theta\right\|_{2} \approx\left\|x_{j}-\Theta\right\|_{2}$.\\
$\therefore$ According to Definition \ref{def:shock force},$\left|\mathcal{F}_{\star}\right| \approx\left|\mathcal{F}_{\diamond}\right|$.\\
$\because$ According to Theorem \ref{the:diffrent block} and Theorem \ref{the:mass}, $\mathcal{M}_{\star}<\mathcal{M}_{\diamond}$.\\
$\therefore\left|S_{\star}\right|=\left|\frac{\mathcal{F}_{\star}^{2} \cdot T^{2}}{\mathcal{M}_{\star}^{2}}\right|>\left|\frac{\mathcal{F}_{\diamond}^{2} \cdot T^{2}}{\mathcal{M}_{\diamond}^{2}}\right|=\left|S_{\diamond}\right|$.\\
$\therefore\left\|\overline{\overline{x_i}}-\Theta\right\|_{2}>\left\|\overline{\overline{x_j}}-\Theta\right\|_{2}$.
\end{proof}

\begin{theorem}
\label{the:sparse}
If $\mathcal{B}_i^\dag$ and $\mathcal{B}_j^\dag$ are close to each other, then $\|\mathcal{B}_i^\dag-\mathcal{B}_j^\dag\|_2 <\|\overline{\overline{\mathcal{B}_i}}^{\dagger}-\overline{\overline{\mathcal{B}_j}}^{\dagger}\|_2$.
\end{theorem}
\begin{proof}
We use geometry to prove this theorem. As shown in Figure \ref{fig:geometry}, we draw a line parallel to $\mathcal{B}_i^\dag\mathcal{B}_j^\dag$ from ${\overline{\overline{\mathcal{B}_j}}}^\dag$ (if ${\overline{\overline{\mathcal{B}_i}}}^\dag$ is closer to $\Theta$, then draw the parallel line from ${\overline{\overline{\mathcal{B}_i}}}^\dag$), and this line intersects $\Theta{\overline{\overline{\mathcal{B}_i}}}^\dag$ at point $a$. In addition, we also draw a perpendicular line from ${\overline{\overline{\mathcal{B}_j}}}^\dag$ which intersects $\Theta{\overline{\overline{\mathcal{B}_i}}}^\dag$ at point $b$.\\
\begin{figure}[h]
  \centering
  \includegraphics[width=3in]{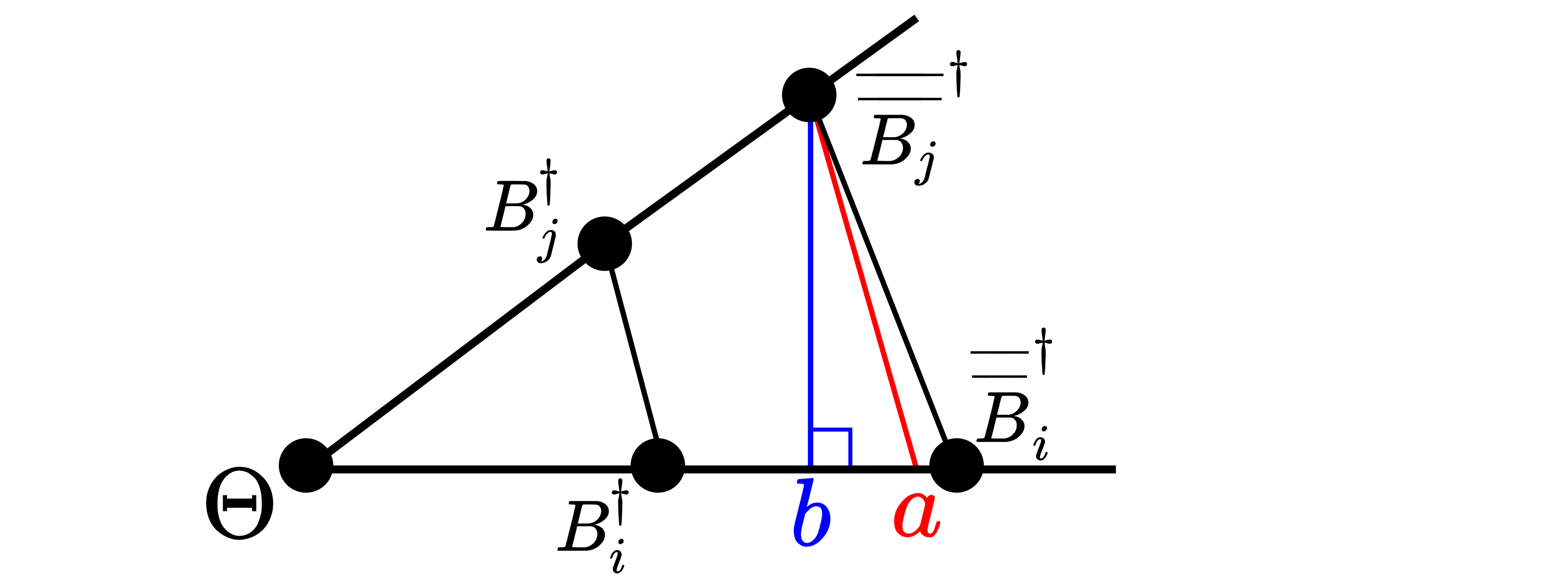}
  \caption{Geometric proof.}
  \label{fig:geometry}
\end{figure}
$\because \mathcal{B}_{i}^{\dagger} \mathcal{B}_{j}^{\dagger} \parallel a\overline{\overline{\mathcal{B}_j}}^{\dagger}$.\\
$\therefore$ According to theorems for the similarity of triangles \cite{ghyka1977geometry}, $\triangle \mathcal{B}_{i}^{\dagger}\Theta\mathcal{B}_{j}^{\dagger} \sim \triangle a\Theta\overline{\overline{\mathcal{B}_j}}^{\dagger}$.\\
$\because\left\|\Theta-\mathcal{B}_{i}^{\dagger}\right\|_{2}<\|\Theta-a\|_{2}$.\\
$\therefore\left\|\mathcal{B}_{i}^{\dagger}-\mathcal{B}_{j}^{\dagger}\right\|_{2}<\left\|a-\overline{\overline{\mathcal{B}_j}}^{\dagger}\right\|_{2}$.\\
$\because \angle b\overline{\overline{\mathcal{B}_j}}^{\dagger}a<\angle b \overline{\overline{\mathcal{B}_j}}^{\dagger}{\overline{\overline{\mathcal{B}_{i}}}}^{\dagger}<90^{\circ}$.\\
$\therefore\left|{\overline{\overline{\mathcal{B}_{j}}}}^{\dagger}b\right| \cdot \sec\angle b{\overline{\overline{\mathcal{B}_{j}}}}^{\dagger} a<\left|{\overline{\overline{\mathcal{B}_{j}}}}^{\dagger}b\right| \cdot \sec\angle b{\overline{\overline{\mathcal{B}_{j}}}}^{\dagger}{\overline{\overline{\mathcal{B}_{i}}}}^{\dagger}$, in which $\left|{\overline{\overline{\mathcal{B}_{j}}}}^{\dagger}b\right|$ is the length of ${\overline{\overline{\mathcal{B}_{j}}}}^{\dagger}b$.\\
$\therefore\left\|a-\overline{\overline{\mathcal{B}_j}}^{\dagger}\right\|_{2}<\left\|{\overline{\overline{\mathcal{B}_{i}}}}^{\dagger}-{\overline{\overline{\mathcal{B}_{j}}}}^{\dagger}\right\|_{2}$.\\
$\therefore\left\|\mathcal{B}_{i}^{\dagger}-\mathcal{B}_{j}^{\dagger}\right\|_{2}<\left\|{\overline{\overline{\mathcal{B}_{i}}}}^{\dagger}-{\overline{\overline{\mathcal{B}_{j}}}}^{\dagger}\right\|_{2}$.
\end{proof}

We describe the detailed implementation of the explosion process in Algorithm \ref{alg:The feature explosion}.
\begin{algorithm}
\caption{The Explosion Process}
\label{alg:The feature explosion}
\KwIn{$X$, $\left\{\mathcal{B}_1,\mathcal{B}_2,\cdots,\mathcal{B}_C\right\}$, $\left\{\mathcal{M}_1,\mathcal{M}_2,\cdots,\mathcal{M}_C\right\}$, $T$}
\KwOut{$\overline{\overline{X}}$, $\left\{\overline{\overline{\mathcal{B}_1}},\overline{\overline{\mathcal{B}_2}},\cdots,\overline{\overline{\mathcal{B}_C}}\right\}$, $\left\{{\overline{\overline{\mathcal{B}_1}}}^\dag,{\overline{\overline{\mathcal{B}_2}}}^\dag,\cdots,{\overline{\overline{\mathcal{B}_C}}}^\dag\right\}$,  $\left\{\overline{\overline{\mathcal{M}_1}},\overline{\overline{\mathcal{M}_2}},\cdots,\overline{\overline{\mathcal{M}_C}}\right\}$}
According to the formula (\ref{eq:virtual bomb}), calculating the virtual bomb $\Theta$.\\
\For{ $\mathcal{B}_i^\dag$ in $\left\{\mathcal{B}_1^\dag,\mathcal{B}_2^\dag,\cdots,\mathcal{B}_C^\dag\right\}$}{
According to the formula (\ref{eq:displacement}), calculating the displacement $S_i$.\\
\For{$x_j$ in $\mathcal{B}_i$}{
According to formula (\ref{eq:object after explosion}), obtaining $\overline{\overline{x_j}}$.
}
Updating $\mathcal{B}_i$ and $\mathcal{B}_i^\dag$ to obtain $\overline{\overline{\mathcal{B}_i}}$ and ${\overline{\overline{\mathcal{B}_i}}}^\dag$. $\overline{\overline{\mathcal{M}_i}}=\mathcal{M}_i$.
}
\Return{$\overline{\overline{X}}$, $\left\{\overline{\overline{\mathcal{B}_1}},\overline{\overline{\mathcal{B}_2}},\cdots,\overline{\overline{\mathcal{B}_C}}\right\}$, $\left\{{\overline{\overline{\mathcal{B}_1}}}^\dag,{\overline{\overline{\mathcal{B}_2}}}^\dag,\cdots,{\overline{\overline{\mathcal{B}_C}}}^\dag\right\}$,  $\left\{\overline{\overline{\mathcal{M}_1}},\overline{\overline{\mathcal{M}_2}},\cdots,\overline{\overline{\mathcal{M}_C}}\right\}$}
\end{algorithm}

\subsection{Repulsion Process}
\label{sec:Repulsion Process}
\textbf{Motivation.} We have proved through Theorem \ref{the:far} that when the outliers and the normal objects are at the same distance from the virtual bomb, the outliers will move farther (see Remark \ref{rem:Far distance} for details). However, according to Definition \ref{def:shock force} and the formula (\ref{eq:displacement}), the object-blocks with larger masses and farther distances from the virtual bomb will have smaller displacements. Therefore, in the same direction of movement, some object-blocks with small masses and close to the virtual bomb may catch up with those object-blocks with large masses and far from the virtual bomb. That is, in the same direction of movement, some outliers may mix into the normal objects, thus misleading the outlier detection algorithms to detect them as normal objects. To solve this problem, in $\overline{\overline{X}}$, we attempt to add repulsive forces among the object-blocks, forcing the object-blocks that are close to each other to separate.

\begin{definition}
\label{def:Invalid neighbors}
\textbf{(Invalid Neighbors)} If in $X$, $x_p$ does not belong to $k$NN neighbors of $x_g$, but in $\overline{\overline{X}}$, $\overline{\overline{x_p}}$ belongs to $k$NN neighbors of $\overline{\overline{x_g}}$, then $\overline{\overline{x_p}}$ is an invalid neighbor of $\overline{\overline{x_g}}$.
\end{definition}

\begin{figure}[h]
  \centering
  \includegraphics[width=3in]{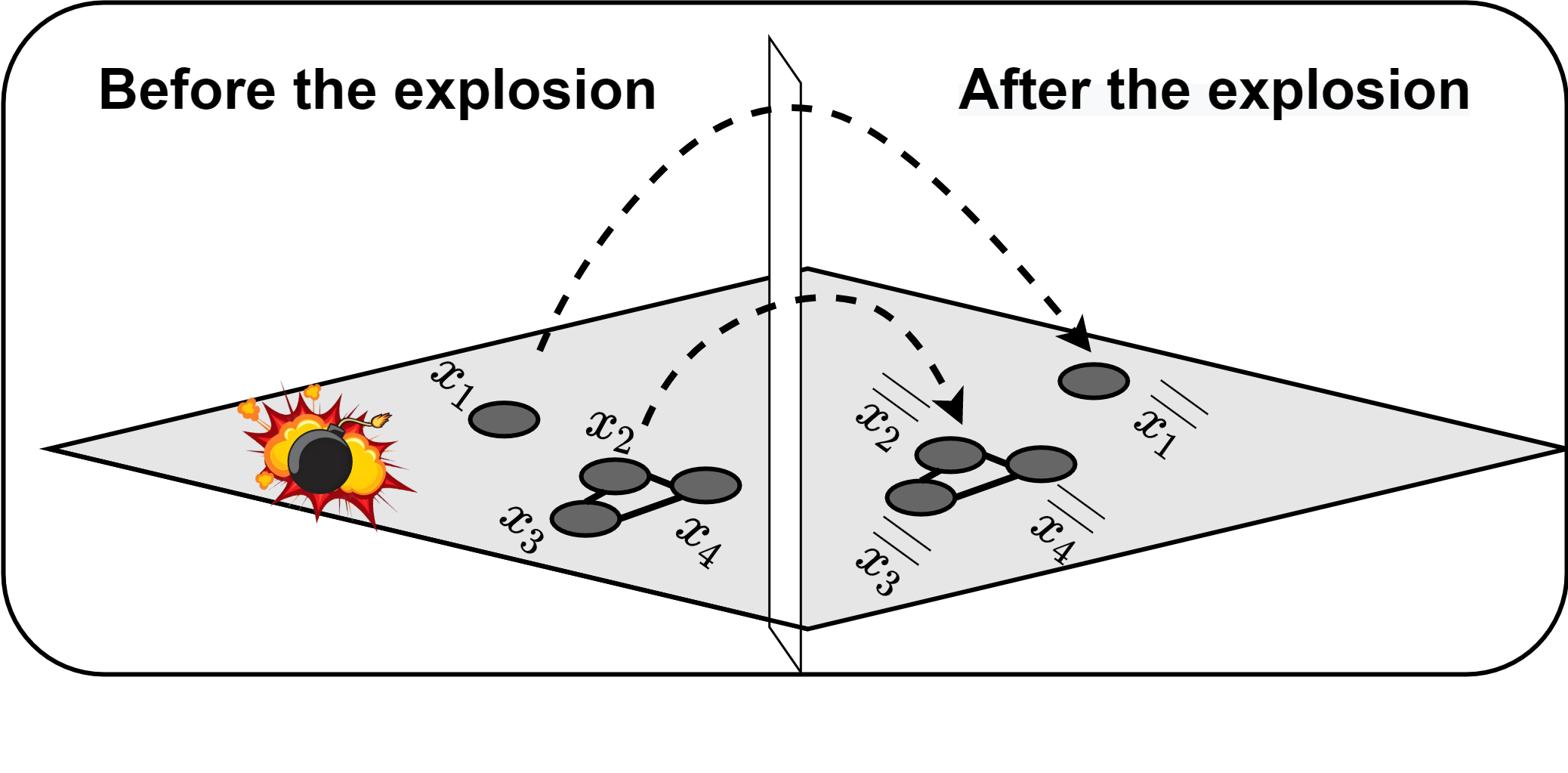}
  \caption{An example of invalid neighbors.}
  \label{fig:Invalid_neighbors}
\end{figure}

\begin{example}
\label{exa:Invalid neighbors}
\textbf{(Invalid Neighbors)} As shown in Figure \ref{fig:Invalid_neighbors}, $x_1$ forms an object-block by itself, and $x_2,x_3,x_4$ form another object-block. Let $k=2$. Before the explosion (i.e., in $X$), compared with $x_4$, $x_2$ and $x_3$ are closer to $x_1$, so $x_4$ does not belong to the $k$NN neighbors of $x_1$. After the explosion (i.e., in $\overline{\overline{X}}$), compared with $\overline{\overline{x_3}}$, $\overline{\overline{x_2}}$ and $\overline{\overline{x_4}}$ are closer to $\overline{\overline{x_1}}$, so $\overline{\overline{x_4}}$ belongs to the $k$NN neighbors of $\overline{\overline{x_1}}$. Therefore, according to Definition \ref{def:Invalid neighbors}, $\overline{\overline{x_4}}$ is an invalid neighbor of $\overline{\overline{x_1}}$. 
\end{example}

In order to reduce unnecessary calculations and movements, we only add repulsive forces between the invalid neighbors among the object-blocks. That is, for $\forall\overline{\overline{x_g}}\in\overline{\overline{\mathcal{B}_i}}$ and $\forall\overline{\overline{x_p}}\in\overline{\overline{\mathcal{B}_j}}$ ($i\neq j$), if $\overline{\overline{x_p}}$ is an invalid neighbor of $\overline{\overline{x_g}}$, then there exists a repulsive force between $\overline{\overline{x_p}}$ and $\overline{\overline{x_g}}$, denoted as $\mathbb{F}_{gp}$,
\begin{equation}
\mathbb{F}_{gp}=\frac{\overline{\overline{x_p}}-\overline{\overline{x_g}}}{\|\overline{\overline{x_p}}-\overline{\overline{x_g}}\|_2}\cdot\frac{1}{\|\overline{\overline{x_p}}-\overline{\overline{x_g}}\|_2},
\label{eq:repulsive force}
\end{equation}
in which $\frac{\overline{\overline{x_p}}-\overline{\overline{x_g}}}{\|\overline{\overline{x_p}}-\overline{\overline{x_g}}\|_2}$ is used to control the direction of the repulsive force, and $\frac{1}{\|\overline{\overline{x_p}}-\overline{\overline{x_g}}\|_2}$ is used to control the scale of the repulsive force. The farther the distance between $\overline{\overline{x_g}}$ and $\overline{\overline{x_p}}$ is, the smaller the repulsive force is. 

Eventually, the resultant force of the repulsive forces exerted on ${\overline{\overline{\mathcal{B}_i}}}^\dag$ is $\mathbb{F}_i$,
\begin{equation}
\mathbb{F}_i=\sum_{\overline{\overline{x_g}}\in\overline{\overline{\mathcal{B}_i}}}\left(\sum_{\overline{\overline{x_p}}\in\varphi_g}\mathbb{F}_{gp}\right),
\label{eq:repulsive force-heli}
\end{equation}
in which $\varphi_g$ is the set of invalid neighbors of $\overline{\overline{x_g}}$. The movement of ${\overline{\overline{\mathcal{B}_i}}}^\dag$ under repulsive forces follows the same principle as in the explosion process, here we do not elaborate further. For $\forall\overline{\overline{x_g}}\in\overline{\overline{\mathcal{B}_i}}$, after the repulsion process, $\overline{\overline{x_g}}$ will transform into $\widehat{x_g}$, 
\begin{equation}
\widehat{x_g}=\overline{\overline{x_g}}+\frac{{\mathbb{F}_i}^2}{{\overline{\overline{\mathcal{M}_i}}}^2}.
\label{eq:object after Repulsion Process}
\end{equation}
We denote $\overline{\overline{X}}$ after the repulsion process as $\widehat{X}$. We describe the detailed implementation of the repulsion process in Algorithm \ref{alg:The repulsion process}.

Finally, outlier detection algorithms can detect outliers from $\widehat{X}$ instead of $X$ to achieve higher accuracy.

\begin{algorithm}
\caption{The Repulsion Process}
\label{alg:The repulsion process}
\KwIn{$\overline{\overline{X}}$, $\left\{\overline{\overline{\mathcal{B}_1}},\overline{\overline{\mathcal{B}_2}},\cdots,\overline{\overline{\mathcal{B}_C}}\right\}$, $\left\{{\overline{\overline{\mathcal{B}_1}}}^\dag,{\overline{\overline{\mathcal{B}_2}}}^\dag,\cdots,{\overline{\overline{\mathcal{B}_C}}}^\dag\right\}$,  $\left\{\overline{\overline{\mathcal{M}_1}},\overline{\overline{\mathcal{M}_2}},\cdots,\overline{\overline{\mathcal{M}_C}}\right\}$}
\KwOut{$\widehat{X}$}
\For{$\overline{\overline{\mathcal{B}_i}},\overline{\overline{\mathcal{B}_j}}$ in $\left\{\overline{\overline{\mathcal{B}_1}},\overline{\overline{\mathcal{B}_2}},\cdots,\overline{\overline{\mathcal{B}_C}}\right\}$}{
Adding repulsive forces according to the formula (\ref{eq:repulsive force}).
}
\For{${\overline{\overline{\mathcal{B}_i}}}^\dag$ in $\left\{{\overline{\overline{\mathcal{B}_1}}}^\dag,{\overline{\overline{\mathcal{B}_2}}}^\dag,\cdots,{\overline{\overline{\mathcal{B}_C}}}^\dag\right\}$}{
Calculating the resultant force according to the formula (\ref{eq:repulsive force-heli}).\\
\For{$\overline{\overline{x_g}}$ in $\overline{\overline{\mathcal{B}_i}}$}{
Calculating $\widehat{x_g}$ according to the formula (\ref{eq:object after Repulsion Process}).
}
}
\Return{$\widehat{X}$}
\end{algorithm}

\begin{table*}
  \caption{The detailed information of datasets.}
  \label{tab:datasets}
  \centering
  \begin{tabular}{cccc|cccc}
    \hline
    \multicolumn{1}{c}{ }&\multicolumn{1}{c}{\emph{Number}}& \multicolumn{1}{c}{\emph{Dimension}}&\multicolumn{1}{c|}{\emph{Proportion}}&\multicolumn{1}{c}{ }&\multicolumn{1}{c}{\emph{Number}}& \multicolumn{1}{c}{\emph{Dimension}}&\multicolumn{1}{c}{\emph{Proportion}}\\
    \hline
    \emph{ALOI} & 49,999 & 27 & 3\% & \emph{Mammography} & 11,183 & 6 & 2.3\%\\
    \emph{Annthyroid} & 7,200 & 6 & 7.42\% & \emph{Cardio} & 1,831 & 21 & 9.6\%\\
    \emph{Speech} & 3,686 & 400 & 1.65\% & \emph{Glass} & 214 & 9 & 4.2\%\\
    \emph{Cardiotocography} & 2,068	& 21 & 20\% & \emph{Ionosphere} & 351 & 33 & 36\%\\
    \emph{Ionosphere\_norm} & 351 & 32 & 35.9\% &	\emph{Letter} & 1,600 & 32 & 6.25\%\\
    \emph{Stamps} & 340 & 9 & 9.12\% & \emph{Lympho} & 148 & 18 & 4.1\%\\
    \emph{WDBC} & 367 & 30 & 2.72\% & \emph{Pima} & 768 & 8 & 35\%\\
    \emph{Waveform} & 3,443 & 21 & 2.9\% & \emph{Thyroid} & 3,772 & 6 & 2.5\%\\
    \emph{HeartDisease} & 187 & 13 & 19.79\% &  \emph{Vowels} & 1,456 & 12 & 3.4\%\\
    \emph{Arrhythmia\_20} & 305 & 259 & 20\% & \emph{Wbc} & 378 & 30 & 5.6\%\\
    \emph{Arrhythmia} & 452 & 274 & 15\% & \emph{Wine} & 129 & 13 & 7.7\%\\
    \emph{Breastw} & 683 & 9 & 35\% & \emph{PageBlocks} & 5,171 & 10 & 4.98\%\\
    \hline
  \end{tabular}
\end{table*}

\subsection{Time Complexity Analysis}
\label{sec:Time Complexity Analysis}
1) The Object-block Division: OSD utilizes KDtree to search for $k$NN neighbors in order to construct $k$-nearest neighbor graph and compute weights, with a time complexity of $O(Nlog(N)+N)$; OSD traverses the dataset to generate the probability distribution curve, and then clip small-weight edges and determines the connected subgraphs (\emph{i.e.}, object-blocks), with a time complexity of $O(3N+E)$, where $E$ is the number of edges. 2) The Explosion Process: OSD computes particles and the virtual bomb, with a time complexity of $O(N)$; OSD traverses the set of particles to compute the explosion shock force and displacement, and traverses each object-block to update the dataset, with a time complexity of $O(C+N)$, where $C$ is the number of object-blocks. 3) The Repulsion Process: OSD utilizes KDtree to search for $k$NN neighbors after explosion, determines the invalid neighbors by comparing with the original $k$NN neighbors, and adds repulsion, with time complexity of $O(Nlog(N))$; OSD computes the resultant force and traverses each object-block to update the dataset, with a time complexity of $O(N)$. In summary, the total time complexity of OSD is $O(2Nlog(N)+7N+E+C)$.

\section{Experiments}
\label{sec:Experiments}
\subsection{Experimental Setting} 
\subsubsection{Datasets} 
We select 24 real-world datasets \cite{Rayana2016}. Table \ref{tab:datasets} records the detailed information of these datasets, including the number of objects, the dimension of datasets, the proportion of outliers.

\subsubsection{Baseline Algorithms}
So far, no outlier detection optimization strategy as generic as OSD has been proposed. In order to test OSD, we compare the performance between the OSD-optimized outlier detection algorithms and their optimized-version algorithms (see Definition \ref{def:optimized-version}). If the performance of the OSD-optimized outlier detection algorithms is superior to that of their optimized-version algorithms, then it indicates that we will not need to spend a huge amount of time and effort on designing new optimized-version algorithms for the existing outlier detection algorithms, instead, we can simply invoke the OSD plugin. In this paper, the selected outlier detection algorithms are two classic baseline algorithms, namely LOF \cite{breunig2000lof} and IForest (\emph{i.e.}, Isolation Forest \cite{liu2012isolation}). The selected optimized-version algorithms of LOF are KNNLOF \cite{xu2022outlier}, CBLOF \cite{he2003discovering}, and COF \cite{tang2002enhancing}; The selected optimized-version algorithms of IForest are EIF \cite{hariri2021extended}, DIF \cite{xu2023deep}, and INNE \cite{bandaragoda2018isolation}.

In addition, we compare the performance of mainstream outlier detection algorithms before and after being optimized by OSD, so as to verify the adaptability of OSD to various outlier detection principles. These mainstream algorithms are deep learning-based LUNAR \cite{goodge2022lunar} and RCA \cite{liu2021rca}, density-based OTF \cite{huang2023novel} and HDIOD \cite{zhou2024outlier}, and statistical-based ECOD \cite{li2023ecod} and BLDOD \cite{aydin2023boundary}. We also compare OSD with its variants, namely OSD-Random (\emph{i.e.}, OSD without virtual bomb positioning), OSD-NoRe (\emph{i.e.}, OSD without the repulsion process), and OSD-NOdiv (\emph{i.e.}, OSD without dividing object-blocks), so as to validate the necessity of the components within OSD.

\subsubsection{Evaluation Metrics} 
In this paper, we select two common evaluation metrics for outlier detection, namely \textbf{AUC} and \textbf{AP} \cite{campos2016evaluation, ok2024exploring}. The ranges of these metrics are all from 0 to 1. The closer the value is to 1, the more accurate the outlier detection is.

\subsection{Comparison Experiments} 
We invoke the OSD plugin to optimize 14 outlier detection algorithms on 24 datasets. Tables \ref{tab:accuracy(AUC)} and \ref{tab:accuracy(AP)} record the accuracy of these outlier detection algorithms before and after optimization. Based on Tables \ref{tab:accuracy(AUC)} and \ref{tab:accuracy(AP)}, we can draw two conclusions:

\textbf{1) The OSD-optimized outlier detection algorithms can replace their optimized-version algorithms.} We calculate the average accuracy of the OSD-optimized IForest (hereafter referred to as IForest+OSD) and IForest’s optimized-version algorithms (\emph{i.e.}, EIF, DIF, INNE) on 24 datasets. Specifically, the average accuracy of EIF is 0.812 (AUC) and 0.302 (AP); The average accuracy of DIF is 0.74 (AUC) and 0.217 (AP); The average accuracy of INNE is 0.805 (AUC) and 0.281 (AP); The average accuracy of IForest+OSD is 0.898 (AUC) and 0.465 (AP). In addition, we also calculate the average accuracy of the OSD-optimized LOF (hereafter referred to as LOF+OSD) and LOF’s optimized-version algorithms (\emph{i.e.}, COF, CBLOF, KNNLOF). Specifically, the average accuracy of COF is 0.703 (AUC) and 0.208 (AP); The average accuracy of CBLOF is 0.783 (AUC) and 0.263 (AP); The average accuracy of KNNLOF is 0.475 (AUC) and 0.128 (AP); The average accuracy of LOF+OSD is 0.863 (AUC) and 0.399 (AP). Obviously, regardless of the evaluation metric, the average accuracies of IForest+OSD and LOF+OSD are higher than those of optimized-version algorithms. In other words, the OSD-optimized outlier detection algorithms can replace their optimized-version algorithms. \emph{Therefore, in the future, we will not need to spend a huge amount of time and effort on designing new optimized-version algorithms for the existing outlier detection algorithms. Instead, we can simply directly invoke the OSD plugin to optimize them.}

\begin{sidewaystable*}
    \centering
    \caption{The accuracy (AUC) of outlier detection algorithms before and after OSD optimization.}
    \label{tab:accuracy(AUC)}
    \setlength{\tabcolsep}{2pt}
\scalebox{0.8}{
\begin{tabular}{ccc|cc|cc|cc|cc|cc|cc|cc|cc|cc|cc|cc|cc|cc}
\hline
~ & \multicolumn{2}{c|}{LOF} & \multicolumn{2}{c|}{COF} & \multicolumn{2}{c|}{CBLOF} & \multicolumn{2}{c|}{KNNLOF} & \multicolumn{2}{c|}{IForest} & \multicolumn{2}{c|}{EIF} & \multicolumn{2}{c|}{INNE} & \multicolumn{2}{c|}{DIF} & \multicolumn{2}{c|}{BLDOD} & \multicolumn{2}{c|}{ECOD} & \multicolumn{2}{c|}{HDIOD} & \multicolumn{2}{c|}{LUNAR} & \multicolumn{2}{c|}{OTF} & \multicolumn{2}{c}{RCA}\\
~ &  & +OSD &  & +OSD &  & +OSD &  & +OSD &  & +OSD &  & +OSD &  & +OSD &  & +OSD &  & +OSD &  & +OSD &  & +OSD &  & +OSD &  & +OSD &  & +OSD\\
\hline
ALOI & 0.783 & \textbf{0.785} & nan & nan & 0.543 & \textbf{0.658} & 0.54 & \textbf{0.581} & 0.55 & \textbf{0.664} & 0.553 & \textbf{0.691} & 0.561 & \textbf{0.653} & 0.55 & \textbf{0.638} & 0.5 & \textbf{0.548} & 0.529 & \textbf{0.645} & 0.753 & \textbf{0.764} & 0.753 & \underline{0.734} & 0.458 & \textbf{0.735} & 0.551 & \textbf{0.639}\\
Arrhythima\_20 & 0.741 & \textbf{0.746} & 0.737 & \textbf{0.765} & 0.711 & \textbf{0.736} & 0.297 & \textbf{0.604} & 0.745 & \textbf{0.788} & 0.715 & \textbf{0.744} & 0.725 & \textbf{0.75} & 0.715 & \textbf{0.75} & 0.5 & \textbf{0.711} & 0.714 & \textbf{0.73} & 0.727 & \textbf{0.74} & 0.739 & \textbf{0.754} & 0.412 & \textbf{0.66} & 0.754 & \textbf{0.794}\\
Cardiotocography & 0.583 & \textbf{0.715} & 0.542 & \textbf{0.618} & 0.681 & \textbf{0.898} & 0.52 & \textbf{0.606} & 0.763 & \textbf{0.875} & 0.76 & \textbf{0.831} & 0.798 & \textbf{0.848} & 0.632 & \textbf{0.875} & 0.5 & \textbf{0.803} & 0.795 & \textbf{0.799} & 0.61 & \textbf{0.629} & 0.555 & \textbf{0.799} & 0.448 & \textbf{0.53} & 0.676 & \textbf{0.88}\\
HeartDisease & 0.557 & \textbf{0.823} & 0.548 & \textbf{0.734} & 0.82 & \textbf{0.825} & 0.494 & \textbf{0.625} & 0.748 & \textbf{0.832} & 0.751 & \textbf{0.828} & 0.675 & \textbf{0.833} & 0.543 & \textbf{0.758} & 0.725 & \textbf{0.762} & 0.731 & \textbf{0.818} & 0.71 & \textbf{0.835} & 0.686 & \textbf{0.81} & 0.444 & \textbf{0.56} & 0.705 & \textbf{0.757}\\
Ionosphere\_norm & 0.905 & \textbf{0.928} & 0.911 & \textbf{0.94} & 0.799 & \textbf{0.915} & 0.385 & \textbf{0.767} & 0.852 & \textbf{0.95} & 0.909 & \textbf{0.946} & 0.902 & \textbf{0.941} & 0.898 & \textbf{0.935} & 0.874 & \textbf{0.928} & 0.728 & \textbf{0.887} & 0.927 & \textbf{0.928} & 0.926 & \textbf{0.934} & 0.506 & \textbf{0.58} & 0.813 & \textbf{0.915}\\
PageBlocks & 0.804 & \textbf{0.839} & 0.717 & \textbf{0.812} & 0.877 & \textbf{0.933} & 0.535 & \textbf{0.632} & 0.912 & \textbf{0.925} & 0.915 & \textbf{0.925} & 0.937 & \textbf{0.939} & 0.909 & \textbf{0.945} & 0.5 & \textbf{0.748} & 0.922 & \textbf{0.936} & 0.84 & \underline{0.839} & 0.83 & \textbf{0.928} & 0.457 & \textbf{0.646} & 0.911 & \textbf{0.912}\\
Stamps & 0.705 & \textbf{0.964} & 0.528 & \textbf{0.865} & 0.914 & \textbf{0.916} & 0.461 & \textbf{0.685} & 0.913 & \textbf{0.953} & 0.88 & \textbf{0.942} & 0.83 & \textbf{0.945} & 0.846 & \textbf{0.971} & 0.515 & \textbf{0.851} & 0.876 & \textbf{0.952} & 0.883 & \textbf{0.895} & 0.866 & \textbf{0.96} & 0.424 & \textbf{0.595} & 0.793 & \textbf{0.956}\\
WDBC & 0.904 & \textbf{0.986} & 0.838 & \textbf{0.991} & 0.945 & \textbf{0.984} & 0.422 & \textbf{0.897} & 0.948 & \textbf{0.976} & 0.943 & \textbf{0.969} & 0.927 & \textbf{0.962} & 0.739 & \textbf{0.978} & 0.935 & \textbf{0.953} & 0.917 & \textbf{0.959} & 0.92 & \textbf{0.926} & 0.928 & \textbf{0.97} & 0.394 & \textbf{0.603} & 0.707 & \textbf{0.91}\\
Waveform & 0.73 & \textbf{0.744} & 0.659 & \textbf{0.705} & 0.592 & \textbf{0.898} & 0.424 & \textbf{0.524} & 0.737 & \textbf{0.817} & 0.768 & \textbf{0.811} & 0.754 & \textbf{0.811} & 0.725 & \textbf{0.791} & 0.5 & \textbf{0.853} & 0.608 & \textbf{0.739} & 0.745 & \textbf{0.746} & 0.748 & \textbf{0.762} & 0.473 & \textbf{0.533} & 0.632 & \textbf{0.842}\\
Annthyroid & 0.707 & \textbf{0.752} & 0.657 & \textbf{0.711} & 0.529 & \textbf{0.783} & 0.453 & \textbf{0.55} & 0.881 & \textbf{0.952} & 0.66 & \textbf{0.79} & 0.685 & \textbf{0.706} & 0.675 & \textbf{0.861} & 0.5 & \textbf{0.592} & 0.789 & \textbf{0.83} & 0.741 & \textbf{0.749} & 0.742 & \textbf{0.843} & 0.489 & \textbf{0.538} & 0.709 & \underline{0.703}\\
Arrhythmia & 0.756 & \textbf{0.762} & 0.722 & \textbf{0.747} & 0.801 & 0.801 & 0.433 & \textbf{0.685} & 0.815 & \underline{0.809} & 0.804 & 0.804 & 0.775 & \textbf{0.78} & 0.809 & \textbf{0.815} & 0.573 & \textbf{0.79} & 0.805 & \underline{0.782} & 0.795 & \underline{0.792} & 0.806 & \textbf{0.81} & 0.535 & \textbf{0.597} & 0.746 & \textbf{0.772}\\
Breastw & 0.55 & \textbf{0.914} & 0.658 & \textbf{0.77} & 0.995 & 0.995 & 0.597 & \textbf{0.667} & 0.989 & \textbf{0.993} & 0.987 & \textbf{0.992} & 0.742 & \textbf{0.983} & 0.765 & \textbf{0.98} & 0.364 & \textbf{0.945} & 0.991 & \underline{0.982} & 0.933 & \textbf{0.98} & 0.978 & \textbf{0.983} & 0.489 & \textbf{0.563} & 0.988 & \textbf{0.993}\\
Cardio & 0.632 & \textbf{0.78} & 0.572 & \textbf{0.706} & 0.865 & \textbf{0.962} & 0.578 & \textbf{0.68} & 0.948 & \textbf{0.962} & 0.944 & \textbf{0.954} & 0.939 & \textbf{0.947} & 0.939 & \textbf{0.956} & 0.5 & \textbf{0.862} & 0.935 & \textbf{0.937} & 0.736 & \textbf{0.791} & 0.735 & \textbf{0.906} & 0.52 & \textbf{0.65} & 0.937 & \textbf{0.967}\\
Glass & 0.809 & \textbf{0.886} & 0.751 & \textbf{0.862} & 0.712 & \textbf{0.882} & 0.518 & \textbf{0.921} & 0.726 & \textbf{0.889} & 0.707 & \textbf{0.885} & 0.818 & \textbf{0.934} & 0.79 & \textbf{0.888} & 0.531 & \textbf{0.787} & 0.621 & \textbf{0.872} & 0.853 & \textbf{0.879} & 0.858 & \textbf{0.912} & 0.496 & \textbf{0.726} & 0.714 & \textbf{0.877}\\
Ionosphere & 0.903 & \textbf{0.933} & 0.913 & \textbf{0.939} & 0.809 & \textbf{0.925} & 0.465 & \textbf{0.793} & 0.866 & \textbf{0.949} & 0.913 & \textbf{0.943} & 0.917 & \textbf{0.943} & 0.899 & \textbf{0.929} & 0.875 & \textbf{0.931} & 0.735 & \textbf{0.901} & 0.944 & \textbf{0.947} & 0.938 & \textbf{0.94} & 0.506 & \textbf{0.572} & 0.824 & \textbf{0.897}\\
Letter & 0.906 & \textbf{0.929} & 0.876 & \textbf{0.911} & 0.626 & \textbf{0.902} & 0.416 & \textbf{0.768} & 0.709 & \textbf{0.908} & 0.658 & \textbf{0.904} & 0.712 & \textbf{0.906} & 0.673 & \textbf{0.903} & 0.415 & \textbf{0.882} & 0.572 & \textbf{0.902} & 0.914 & \textbf{0.934} & 0.903 & \textbf{0.946} & 0.46 & \textbf{0.603} & 0.782 & \textbf{0.902}\\
Lympho & 0.98 & \textbf{0.982} & 0.925 & \textbf{0.969} & 0.993 & \textbf{1} & 0.424 & \textbf{0.827} & 1 & 1 & 1 & 1 & 0.993 & \textbf{0.999} & 0.884 & \textbf{1} & 0.992 & \textbf{0.996} & 0.996 & \textbf{0.998} & 0.987 & 0.987 & 0.985 & \textbf{1} & 0.196 & \textbf{0.63} & 0.969 & \textbf{1}\\
Mammography & 0.735 & \textbf{0.792} & 0.717 & \textbf{0.786} & 0.868 & \textbf{0.927} & 0.632 & \textbf{0.694} & 0.889 & \textbf{0.894} & 0.842 & \textbf{0.877} & 0.778 & \textbf{0.865} & 0.756 & \textbf{0.777} & 0.5 & \textbf{0.622} & 0.906 & \textbf{0.911} & 0.838 & 0.838 & 0.848 & \textbf{0.88} & 0.564 & \textbf{0.614} & 0.844 & \underline{0.74}\\
Pima & 0.593 & \textbf{0.672} & 0.576 & \textbf{0.643} & 0.713 & \textbf{0.761} & 0.492 & \textbf{0.577} & 0.69 & \textbf{0.752} & 0.712 & \textbf{0.74} & 0.707 & \textbf{0.723} & 0.628 & \textbf{0.74} & 0.5 & \textbf{0.657} & 0.594 & \textbf{0.696} & 0.702 & \textbf{0.714} & 0.718 & \textbf{0.739} & 0.501 & \textbf{0.564} & 0.726 & \textbf{0.737}\\
Speech & 0.777 & \textbf{0.891} & 0.743 & \textbf{0.776} & 0.461 & \textbf{0.753} & 0.383 & \textbf{0.755} & 0.548 & \textbf{0.735} & 0.5 & \textbf{0.735} & 0.486 & \textbf{0.694} & 0.515 & \textbf{0.733} & 0.501 & \textbf{0.725} & 0.47 & \textbf{0.7} & 0.556 & \textbf{0.66} & 0.564 & \textbf{0.735} & 0.654 & \textbf{0.802} & 0.416 & \textbf{0.734}\\
Thyroid & 0.782 & \textbf{0.949} & 0.597 & \textbf{0.858} & 0.894 & \textbf{0.978} & 0.452 & \textbf{0.692} & 0.984 & \textbf{0.985} & 0.936 & \textbf{0.962} & 0.954 & \textbf{0.962} & 0.962 & \textbf{0.984} & 0.5 & \textbf{0.868} & 0.977 & 0.977 & 0.951 & \textbf{0.955} & 0.952 & \textbf{0.973} & 0.422 & \textbf{0.572} & 0.95 & \underline{0.947}\\
Vowels & 0.951 & \textbf{0.995} & 0.87 & \textbf{0.995} & 0.75 & \textbf{0.992} & 0.519 & \textbf{0.94} & 0.789 & \textbf{0.989} & 0.837 & \textbf{0.987} & 0.913 & \textbf{0.99} & 0.789 & \textbf{0.995} & 0.5 & \textbf{0.978} & 0.593 & \textbf{0.977} & 0.985 & \textbf{0.995} & 0.918 & \textbf{0.995} & 0.55 & \textbf{0.775} & 0.891 & \textbf{0.976}\\
Wbc & 0.947 & \underline{0.946} & 0.741 & \textbf{0.907} & 0.951 & \textbf{0.963} & 0.355 & \textbf{0.728} & 0.958 & \textbf{0.966} & 0.95 & \textbf{0.957} & 0.942 & \textbf{0.97} & 0.727 & \textbf{0.953} & 0.94 & \textbf{0.947} & 0.9 & \textbf{0.916} & 0.949 & \underline{0.948} & 0.949 & \textbf{0.966} & 0.54 & \textbf{0.642} & 0.908 & \textbf{0.944}\\
Wine & 0.857 & \textbf{0.992} & 0.371 & \textbf{0.954} & 0.933 & \textbf{0.984} & 0.602 & \textbf{0.692} & 0.834 & \textbf{0.981} & 0.841 & \textbf{0.987} & 0.842 & \textbf{0.987} & 0.388 & \textbf{0.976} & 0.888 & \textbf{0.95} & 0.733 & \textbf{0.987} & 0.859 & \textbf{0.987} & 0.72 & \textbf{0.942} & 0.501 & \textbf{0.607} & 0.94 & \textbf{0.974}\\
\hline

\end{tabular}
}
\end{sidewaystable*}

\begin{sidewaystable*}
    \centering
    \caption{The accuracy (AP) of outlier detection algorithms before and after OSD optimization.}
    \label{tab:accuracy(AP)}
    \setlength{\tabcolsep}{2pt}
\scalebox{0.8}{
\begin{tabular}{ccc|cc|cc|cc|cc|cc|cc|cc|cc|cc|cc|cc|cc|cc}
\hline
~ & \multicolumn{2}{c|}{LOF} & \multicolumn{2}{c|}{COF} & \multicolumn{2}{c|}{CBLOF} & \multicolumn{2}{c|}{KNNLOF} & \multicolumn{2}{c|}{IForest} & \multicolumn{2}{c|}{EIF} & \multicolumn{2}{c|}{INNE} & \multicolumn{2}{c|}{DIF} & \multicolumn{2}{c|}{BLDOD} & \multicolumn{2}{c|}{ECOD} & \multicolumn{2}{c|}{HDIOD} & \multicolumn{2}{c|}{LUNAR} & \multicolumn{2}{c|}{OTF} & \multicolumn{2}{c}{RCA}\\
~ &  & +OSD &  & +OSD &  & +OSD &  & +OSD &  & +OSD &  & +OSD &  & +OSD &  & +OSD &  & +OSD &  & +OSD &  & +OSD &  & +OSD &  & +OSD &  & +OSD\\
\hline
ALOI & 0.131 & \textbf{0.142} & nan & nan & 0.043 & \textbf{0.087} & 0.04 & \textbf{0.07} & 0.034 & \textbf{0.095} & 0.032 & \textbf{0.067} & 0.042 & \textbf{0.087} & 0.046 & \textbf{0.085} & 0.03 & \textbf{0.067} & 0.032 & \textbf{0.086} & 0.126 & \textbf{0.131} & 0.117 & \underline{0.114} & 0.028 & \textbf{0.051} & 0.033 & \textbf{0.075}\\
Arrhythmia\_20
 & 0.331 & \textbf{0.357} & 0.307 & \textbf{0.357} & 0.286 & \textbf{0.37} & 0.191 & \textbf{0.267} & 0.331 & \textbf{0.43} & 0.307 & \textbf{0.384} & 0.357 & \textbf{0.384} & 0.319 & \textbf{0.37} & 0.2 & \textbf{0.344} & 0.319 & \textbf{0.37} & 0.338 & \textbf{0.362} & 0.331 & \textbf{0.37} & 0.309 & \textbf{0.36} & 0.28 & \textbf{0.422}\\
Cardiotocography
 & 0.237 & \textbf{0.307} & 0.215 & \textbf{0.268} & 0.285 & \textbf{0.543} & 0.21 & \textbf{0.262} & 0.312 & \textbf{0.452} & 0.337 & \textbf{0.486} & 0.356 & \textbf{0.414} & 0.276 & \textbf{0.472} & 0.2 & \textbf{0.396} & 0.341 & \textbf{0.35} & 0.26 & \textbf{0.298} & 0.243 & \textbf{0.362} & 0.224 & \textbf{0.292} & 0.289 & \textbf{0.442}\\
HeartDisease
 & 0.202 & \textbf{0.36} & 0.196 & \textbf{0.299} & 0.383 & 0.383 & 0.192 & \textbf{0.266} & 0.282 & \textbf{0.408} & 0.282 & \textbf{0.383} & 0.217 & \textbf{0.36} & 0.189 & \textbf{0.318} & 0.299 & \textbf{0.338} & 0.239 & \textbf{0.36} & 0.219 & \textbf{0.369} & 0.227 & \textbf{0.36} & 0.242 & \textbf{0.387} & 0.328 & \textbf{0.42}\\
Ionosphere\_norm
 & 0.744 & \textbf{0.775} & 0.765 & \textbf{0.797} & 0.474 & \textbf{0.765} & 0.342 & \textbf{0.604} & 0.588 & \textbf{0.87} & 0.657 & \textbf{0.776} & 0.699 & \textbf{0.799} & 0.675 & \textbf{0.831} & 0.588 & \textbf{0.765} & 0.451 & \textbf{0.765} & 0.805 & \textbf{0.833} & 0.775 & \textbf{0.819} & 0.517 & \textbf{0.792} & 0.512 & \textbf{0.775}\\
PageBlocks
 & 0.21 & \textbf{0.279} & 0.149 & \textbf{0.201} & 0.118 & \textbf{0.359} & 0.077 & \textbf{0.155} & 0.175 & \textbf{0.305} & 0.264 & \textbf{0.309} & 0.363 & \textbf{0.383} & 0.21 & \textbf{0.685} & 0.05 & \textbf{0.237} & 0.181 & \textbf{0.338} & 0.145 & 0.145 & 0.223 & \textbf{0.309} & 0.109 & \textbf{0.251} & 0.196 & \textbf{0.245}\\
Stamps & 0.111 & \textbf{0.488} & 0.111 & \textbf{0.311} & 0.229 & \textbf{0.449} & 0.092 & \textbf{0.254} & 0.149 & \textbf{0.53} & 0.122 & \textbf{0.53} & 0.122 & \textbf{0.352} & 0.102 & \textbf{0.574} & 0.089 & \textbf{0.254} & 0.149 & \textbf{0.488} & 0.144 & \textbf{0.197} & 0.166 & \textbf{0.411} & 0.101 & \textbf{0.455} & 0.161 & \textbf{0.579}\\
WDBC & 0.371 & \textbf{0.645} & 0.035 & \textbf{0.645} & 0.371 & \textbf{0.645} & 0.036 & \textbf{0.411} & 0.645 & 0.645 & 0.498 & \textbf{0.645} & 0.498 & \textbf{0.717} & 0.035 & \textbf{0.645} & 0.066 & \textbf{0.717} & 0.264 & \textbf{0.645} & 0.151 & 0.151 & 0.371 & \textbf{0.645} & 0.211 & \textbf{0.553} & 0.132 & 0.132\\
Waveform & 0.029 & \textbf{0.048} & 0.029 & \textbf{0.048} & 0.029 & \textbf{0.048} & 0.029 & \textbf{0.039} & 0.029 & \textbf{0.048} & 0.044 & \textbf{0.078} & 0.034 & \underline{0.029} & 0.034 & \textbf{0.048} & 0.029 & 0.029 & 0.029 & \textbf{0.048} & 0.064 & \textbf{0.065} & 0.034 & \textbf{0.048} & 0.031 & \textbf{0.069} & 0.033 & \textbf{0.063}\\
Annthyroid & 0.104 & \textbf{0.149} & 0.091 & \textbf{0.117} & 0.076 & \textbf{0.183} & 0.073 & \textbf{0.085} & 0.192 & \textbf{0.276} & 0.111 & \textbf{0.145} & 0.129 & \textbf{0.163} & 0.128 & \textbf{0.14} & 0.074 & \textbf{0.112} & 0.146 & \textbf{0.159} & 0.123 & \textbf{0.125} & 0.127 & \textbf{0.178} & 0.099 & \textbf{0.128} & 0.123 & \textbf{0.145}\\
Arrhythmia & 0.249 & \textbf{0.269} & 0.259 & \textbf{0.316} & 0.303 & \textbf{0.328} & 0.144 & \textbf{0.272} & 0.342 & \underline{0.316} & 0.303 & \textbf{0.342} & 0.259 & \textbf{0.269} & 0.292 & \textbf{0.342} & 0.147 & \textbf{0.232} & 0.316 & \underline{0.259} & 0.28 & \textbf{0.319} & 0.292 & \textbf{0.316} & 0.254 & \textbf{0.29} & 0.208 & \textbf{0.293}\\
Breastw & 0.375 & \textbf{0.746} & 0.365 & \textbf{0.513} & 0.92 & \underline{0.913} & 0.342 & \textbf{0.435} & 0.894 & \textbf{0.92} & 0.888 & \textbf{0.913} & 0.421 & \textbf{0.913} & 0.462 & \textbf{0.885} & 0.35 & \textbf{0.888} & 0.888 & \underline{0.881} & 0.686 & \textbf{0.871} & 0.875 & \textbf{0.888} & 0.471 & \textbf{0.568} & 0.902 & 0.902\\
Cardio & 0.135 & \textbf{0.254} & 0.117 & \textbf{0.22} & 0.347 & \textbf{0.46} & 0.13 & \textbf{0.209} & 0.395 & \textbf{0.498} & 0.37 & \textbf{0.447} & 0.336 & \textbf{0.384} & 0.407 & \textbf{0.587} & 0.096 & \textbf{0.305} & 0.325 & \textbf{0.353} & 0.207 & \textbf{0.25} & 0.161 & \textbf{0.283} & 0.216 & \textbf{0.274} & 0.381 & \textbf{0.637}\\
Glass & 0.082 & \textbf{0.327} & 0.082 & \textbf{0.153} & 0.05 & \textbf{0.139} & 0.056 & \textbf{0.107} & 0.05 & \textbf{0.082} & 0.05 & \textbf{0.082} & 0.05 & \textbf{0.327} & 0.05 & \textbf{0.139} & 0.051 & 0.051 & 0.05 & 0.05 & 0.12 & \textbf{0.167} & 0.05 & \textbf{0.139} & 0.047 & \textbf{0.18} & 0.105 & \textbf{0.201}\\
Ionosphere & 0.765 & \textbf{0.808} & 0.765 & \textbf{0.83} & 0.48 & \textbf{0.797} & 0.367 & \textbf{0.657} & 0.596 & \textbf{0.845} & 0.666 & \textbf{0.801} & 0.744 & \textbf{0.819} & 0.648 & \textbf{0.804} & 0.588 & \textbf{0.797} & 0.456 & \textbf{0.797} & 0.81 & \textbf{0.853} & 0.797 & \textbf{0.83} & 0.476 & \textbf{0.811} & 0.512 & \textbf{0.775}\\
Letter & 0.31 & \textbf{0.396} & 0.263 & \textbf{0.396} & 0.068 & \textbf{0.331} & 0.062 & \textbf{0.254} & 0.076 & \textbf{0.366} & 0.069 & \textbf{0.341} & 0.114 & \textbf{0.42} & 0.065 & \textbf{0.341} & 0.062 & \textbf{0.331} & 0.065 & \textbf{0.32} & 0.294 & \textbf{0.338} & 0.229 & \textbf{0.385} & 0.191 & \textbf{0.366} & 0.088 & \textbf{0.358}\\
Lympho & 0.394 & 0.394 & 0.235 & \textbf{0.458} & 0.602 & \textbf{0.857} & 0.041 & \textbf{0.062} & 0.857 & 0.857 & 1 & 1 & 0.602 & \textbf{0.857} & 0.394 & \textbf{0.857} & 0.458 & \textbf{0.701} & 0.857 & 0.857 & 0.68 & 0.68 & 0.602 & \textbf{0.857} & 0.395 & \textbf{0.547} & 0.289 & \textbf{0.55}\\
Mammography & 0.027 & \textbf{0.081} & 0.027 & \textbf{0.066} & 0.024 & \textbf{0.047} & 0.023 & \textbf{0.057} & 0.087 & \textbf{0.093} & 0.028 & \textbf{0.075} & 0.06 & \textbf{0.079} & 0.051 & \textbf{0.055} & 0.023 & \textbf{0.052} & 0.099 & \underline{0.06} & 0.055 & \textbf{0.056} & 0.055 & \textbf{0.081} & 0.025 & \textbf{0.036} & 0.073 & \textbf{0.078}\\
Pima & 0.388 & \textbf{0.446} & 0.377 & \textbf{0.452} & 0.441 & \textbf{0.48} & 0.352 & \textbf{0.393} & 0.449 & \textbf{0.508} & 0.444 & \textbf{0.493} & 0.452 & \textbf{0.466} & 0.377 & \textbf{0.483} & 0.364 & \textbf{0.417} & 0.397 & \textbf{0.466} & 0.385 & \textbf{0.443} & 0.457 & \textbf{0.486} & 0.434 & \textbf{0.465} & 0.474 & \textbf{0.502}\\
Speech & 0.032 & \textbf{0.108} & 0.036 & \textbf{0.099} & 0.017 & \textbf{0.081} & 0.017 & \textbf{0.074} & 0.02 & \textbf{0.108} & 0.018 & \textbf{0.099} & 0.017 & \textbf{0.032} & 0.02 & \textbf{0.108} & 0.017 & \textbf{0.025} & 0.017 & \textbf{0.081} & 0.019 & \textbf{0.095} & 0.022 & \textbf{0.099} & 0.067 & \textbf{0.126} & 0.018 & \textbf{0.069}\\
Thyroid & 0.053 & \textbf{0.204} & 0.025 & \textbf{0.069} & 0.046 & \textbf{0.262} & 0.025 & \textbf{0.156} & 0.416 & \textbf{0.43} & 0.12 & \textbf{0.244} & 0.297 & \textbf{0.403} & 0.232 & \textbf{0.329} & 0.026 & \textbf{0.225} & 0.317 & \underline{0.294} & 0.137 & \textbf{0.139} & 0.106 & \textbf{0.283} & 0.209 & \textbf{0.273} & 0.178 & \textbf{0.373}\\
Vowel & 0.181 & \textbf{0.779} & 0.166 & \textbf{0.794} & 0.035 & \textbf{0.779} & 0.038 & \textbf{0.628} & 0.084 & \textbf{0.811} & 0.138 & \textbf{0.779} & 0.166 & \textbf{0.849} & 0.061 & \textbf{0.811} & 0.034 & \textbf{0.76} & 0.061 & \textbf{0.779} & 0.359 & \textbf{0.367} & 0.103 & \textbf{0.794} & 0.294 & \textbf{0.763} & 0.201 & \textbf{0.684}\\
Wbc & 0.335 & \textbf{0.387} & 0.096 & \textbf{0.288} & 0.387 & \textbf{0.503} & 0.081 & \textbf{0.256} & 0.443 & 0.443 & 0.404 & 0.404 & 0.335 & \textbf{0.636} & 0.055 & \textbf{0.443} & 0.121 & \textbf{0.35} & 0.207 & \textbf{0.288} & 0.248 & \textbf{0.257} & 0.335 & \textbf{0.443} & 0.253 & \textbf{0.32} & 0.208 & \textbf{0.292}\\
Wine & 0.08 & \textbf{0.818} & 0.078 & \textbf{0.513} & 0.289 & \textbf{0.818} & 0.106 & \textbf{0.154} & 0.102 & \textbf{0.818} & 0.102 & \textbf{0.818} & 0.078 & \textbf{0.908} & 0.078 & \textbf{0.818} & 0.154 & \textbf{0.908} & 0.08 & \textbf{0.818} & 0.129 & \textbf{0.413} & 0.207 & \textbf{0.513} & 0.078 & \textbf{0.586} & 0.258 & \textbf{0.819}\\
\hline
\end{tabular}
}
\end{sidewaystable*}

\textbf{2) OSD is applicable to various outlier detection principles.} In Tables \ref{tab:accuracy(AUC)} and \ref{tab:accuracy(AP)}, if the accuracy after optimization improves, the result is \textbf{bolded}; if the accuracy after optimization decreases, the result is \underline{underlined}. The results show that, for 14 outlier detection algorithms with different principles, OSD can improve the accuracies of all outlier detection algorithms on the majority of datasets. For example, the accuracy of BLDOD on \emph{Breastw} is 0.364 (AUC), but the accuracy of BLDOD+OSD is as high as 0.945 (AUC), with an improvement rate of 159.6\%; the accuracy of COF on \emph{WDBC} is only 0.035 (AP), but the accuracy of COF+OSD reaches 0.645 (AP), with an improvement rate of 1,742.8\%. \emph{Obviously, OSD is not picky about outlier detection principles and is applicable to diverse outlier detection algorithms.} In terms of average accuracy, OSD improves these algorithms by an average of 15\% (AUC) and 63.7\% (AP).

\begin{table*}
  \caption{The comparison between OSD and OSD-Random.}
  \label{tab:OSD-Random}
  \setlength{\tabcolsep}{0.7mm}
  \centering
  \begin{tabular}{cc|cccccccccccccc}
    \hline
    \multicolumn{1}{c}{ }&\multicolumn{1}{c|}{ }&\multicolumn{1}{c}{LOF}& \multicolumn{1}{c}{ECOD}&\multicolumn{1}{c}{COF}&\multicolumn{1}{c}{EIF}&\multicolumn{1}{c}{CBLOF}& \multicolumn{1}{c}{RCA}&\multicolumn{1}{c}{DIF}&\multicolumn{1}{c}{INNE}&\multicolumn{1}{c}{IForest}&\multicolumn{1}{c}{KNNLOF}&\multicolumn{1}{c}{BLDOD}&\multicolumn{1}{c}{HDIOD}&\multicolumn{1}{c}{OTF}&\multicolumn{1}{c}{LUNAR}\\
    \hline
   \multirow{2}{*}{\textbf{AUC}}& IOD  & \textbf{0.863} & \textbf{0.868} & \textbf{0.825} & \textbf{0.883} & \textbf{0.89} & \textbf{0.865} & \textbf{0.88} & \textbf{0.878} & \textbf{0.898} & \textbf{0.704} & \textbf{0.82} & \textbf{0.852} & \textbf{0.621} & \textbf{0.884} \\ 
 & IOD-Random & 0.791 & 0.789 & 0.711 & 0.807 & 0.765 & 0.758 & 0.668 & 0.799 & 0.807 & 0.489 & 0.545 & 0.8 & 0.454 & 0.805 \\ 
\multirow{2}{*}{\textbf{AP}} & IOD & \textbf{0.399} & \textbf{0.413} & \textbf{0.357} & \textbf{0.443} & \textbf{0.443} & \textbf{0.41} & \textbf{0.465} & \textbf{0.46} & \textbf{0.465} & \textbf{0.255} & \textbf{0.388} & 0.33 & \textbf{0.373} & \textbf{0.417} \\ 
 & IOD-Random & 0.387 & 0.381 & 0.306 & 0.414 & 0.376 & 0.282 & 0.282 & 0.424 & 0.425 & 0.167 & 0.23 & \textbf{0.426} & 0.118 & 0.372 \\ 
    \hline
  \end{tabular}
\end{table*}

\subsection{Ablation Experiment}
\textbf{Virtual Bomb Location.} As discussed in Section \ref{sec:The explosion process}, we have proven through Theorem \ref{the:bomb location} that placing the virtual bomb at the center of particles can exert the maximum explosion shock force on the dataset, thereby maximizing the effect of the explosion process. Here, we further validate this conclusion through experiments. OSD-Random is a version of OSD that randomly places the virtual bomb. Table \ref{tab:OSD-Random} records the average accuracies of the outlier detection algorithms optimized by OSD and OSD-Random respectively. The optimal results are \textbf{bolded}. The experimental results show that OSD is significantly better than OSD-Random. The average accuracies of the outlier detection algorithms optimized by OSD are almost always higher than those of the outlier detection algorithms optimized by OSD-Random. For example, the average accuracy of DIF+OSD-Random is only 0.668 (AUC), but the average accuracy of DIF+OSD is as high as 0.88 (AUC). Obviously, placing the virtual bomb at the center of particles is beneficial for the explosion process.

\begin{figure*}
  \centering
  \includegraphics[width=6in]{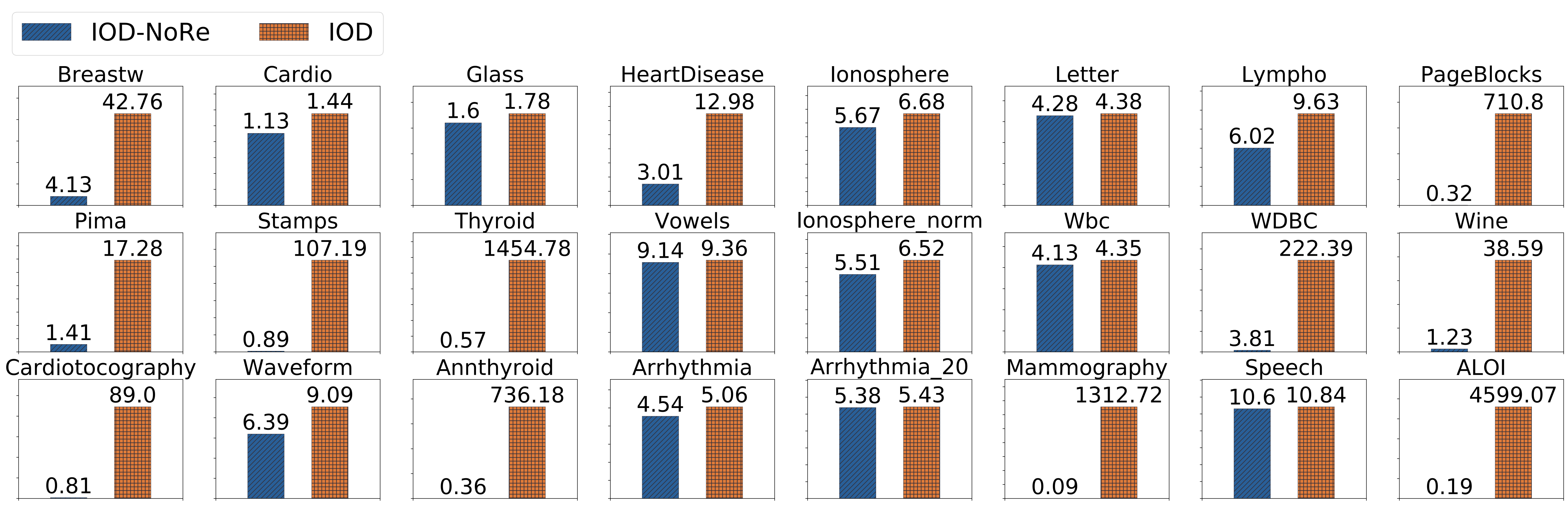}
  \caption{The comparison in distance.}
  \label{fig:OSD-NoRe}
\end{figure*}

\textbf{Repulsion Process.} In Section \ref{sec:Repulsion Process}, we design a repulsion process to prevent some outliers from mixing into normal objects. Here, we conduct experiments to verify whether the repulsion process plays a role. We refer to the version of OSD without the repulsion process as OSD-NoRe. We compare the average distances between outliers and normal objects in the datasets modified by OSD and OSD-NoRe respectively, as shown in Figure \ref{fig:OSD-NoRe}. Results show that the average distances in the datasets modified by OSD are significantly larger than the average distances in the datasets modified by OSD-NoRe. Therefore, the repulsion process can further increase the distance between outliers and normal objects. Table \ref{tab:OSD-NoRe} records the average accuracies of the outlier detection algorithms optimized by OSD and OSD-NoRe respectively. OSD is completely superior to OSD-NoRe, which is sufficient to illustrate that the repulsion process plays a role.

\begin{table*}
  \caption{The comparison in accuracy between OSD and OSD-NoRe.}
  \label{tab:OSD-NoRe}
  \setlength{\tabcolsep}{0.7mm}
  \centering
  \begin{tabular}{cc|cccccccccccccc}
    \hline
    \multicolumn{1}{c}{ }&\multicolumn{1}{c|}{ }&\multicolumn{1}{c}{LOF}& \multicolumn{1}{c}{ECOD}&\multicolumn{1}{c}{COF}&\multicolumn{1}{c}{EIF}&\multicolumn{1}{c}{CBLOF}& \multicolumn{1}{c}{RCA}&\multicolumn{1}{c}{DIF}&\multicolumn{1}{c}{INNE}&\multicolumn{1}{c}{IForest}&\multicolumn{1}{c}{KNNLOF}&\multicolumn{1}{c}{BLDOD}&\multicolumn{1}{c}{HDIOD}&\multicolumn{1}{c}{OTF}&\multicolumn{1}{c}{LUNAR}\\
    \hline
   \multirow{2}{*}{\textbf{AUC}}& IOD  & \textbf{0.863} & \textbf{0.868} & \textbf{0.825} & \textbf{0.883} & \textbf{0.89} & \textbf{0.865} & \textbf{0.88} & \textbf{0.878} & \textbf{0.898} & \textbf{0.704} & \textbf{0.82} & \textbf{0.852} & \textbf{0.621} & \textbf{0.884} \\ 
& OSD-NoRe & 0.756 & 0.817 & 0.723 & 0.839 & 0.786 & 0.824 & 0.786 & 0.819 & 0.853 & 0.678 & 0.789 & 0.831 & 0.586 & 0.823 \\ 
\multirow{2}{*}{\textbf{AP}} & IOD & \textbf{0.399} & \textbf{0.413} & \textbf{0.357} & \textbf{0.443} & \textbf{0.443} & \textbf{0.41} & \textbf{0.465} & \textbf{0.46} & \textbf{0.465} & \textbf{0.255} & \textbf{0.388} & \textbf{0.33} & \textbf{0.373} & \textbf{0.417} \\ 
& OSD-NoRe & 0.288 & 0.329 & 0.26 & 0.313 & 0.33 & 0.292 & 0.303 & 0.328 & 0.335 & 0.222 & 0.288 & 0.308 & 0.283 & 0.305 \\ 
    \hline
  \end{tabular}
\end{table*}

\subsection{Robustness Experiments}
\label{sec:Robustness experiments}
\textbf{Inflection Point.} As described in Section \ref{sec:The object-block division process}, OSD treats the inflection point as the threshold to divide object-blocks. Since the inflection point is actually a region rather than a value, in practice, we randomly select a value from the region as the threshold. Therefore, it is necessary to discuss the impact of selecting different values within the inflection point region on OSD. Here, we test OSD on datasets \emph{T8.8k} and \emph{Worm} \cite{ClusteringDatasets}, in which outliers intersperse between clusters. The first column of Figure \ref{fig:inflection} shows the original distribution of these datasets. We select three different values as thresholds from the inflection point region, as indicated by the red stars in the second column of Figure \ref{fig:inflection}. Columns 3 to 6 of Figure \ref{fig:inflection} show the datasets modified by OSD under the three different thresholds. Since the distribution of objects becomes broader after the explosion process, we specifically enlarge the view of the area where clusters are located (\emph{i.e.}, the area where normal objects are located). The experimental results show that, regardless of the threshold selected, the vast majority outliers are far from clusters (\emph{i.e.}, are far from normal objects). This result is expected because each object is connected by edges only to its $k$NN neighbors, meaning there are very rare outlier-edges (as defined in Lemma \ref{lemma}). The rarity causes the weights of outlier-edges to be concentrated on the left side of the inflection point region. Therefore, regardless of the threshold selected, outlier-edges will be cut. In summary, randomly selecting a value from the inflection point region as a threshold is robust for OSD.

\begin{figure*}
  \centering
    \includegraphics[width=6in]{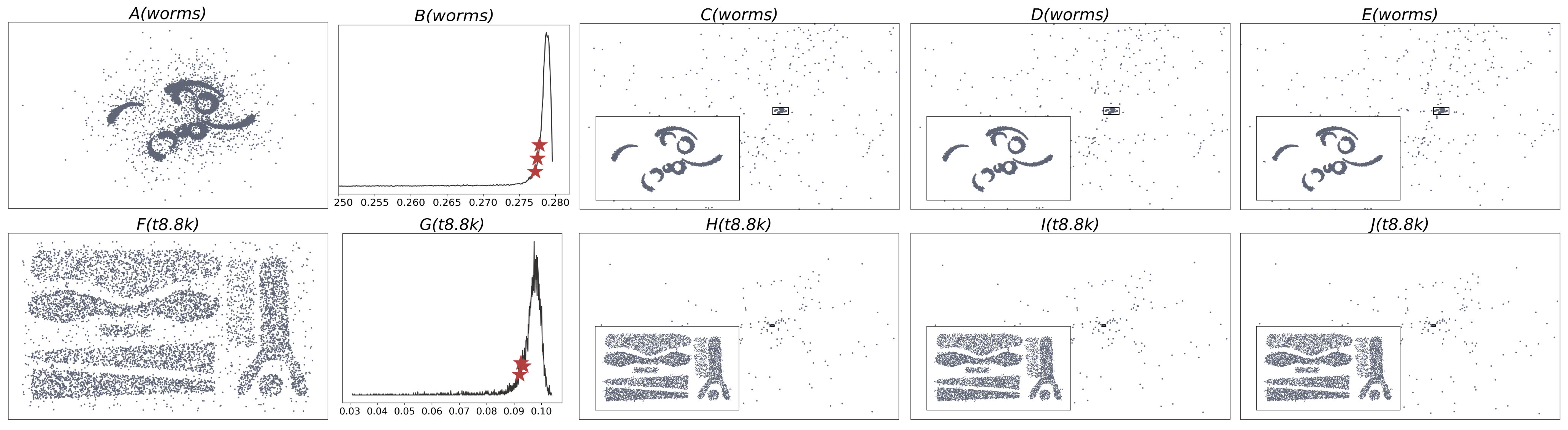}\\
  \caption{The impact of different values within the inflection point region on OSD.}
  \label{fig:inflection}
\end{figure*}

\begin{figure*}
  \centering
  \includegraphics[width=6in]{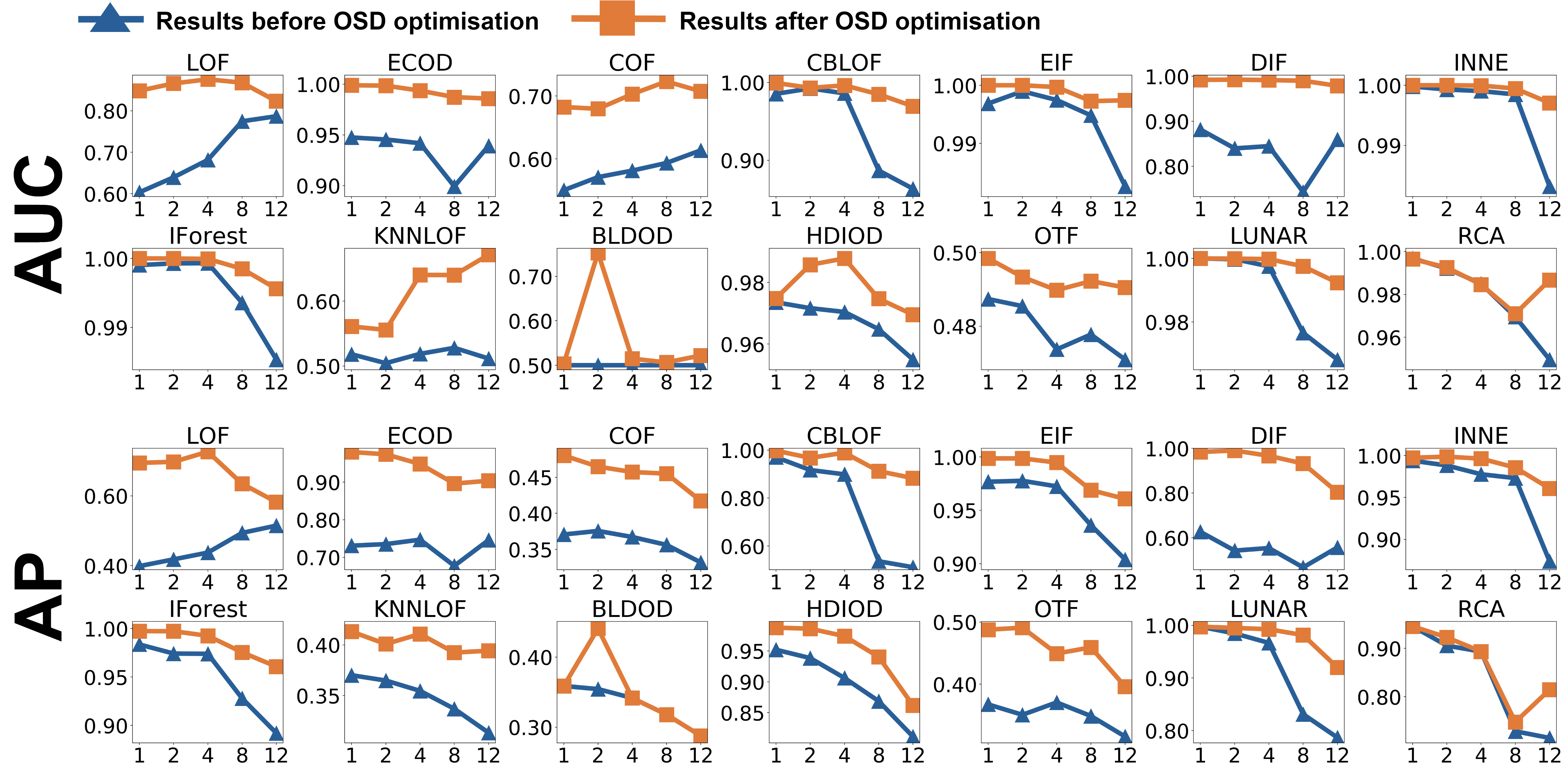}
  \caption{The impact of imbalance density on OSD.}
  \label{fig:unbalance}
\end{figure*}

\textbf{Imbalance Density.} Imbalance density is a common factor that interferes with outlier detection algorithms. In density-imbalanced datasets, where the density differences between normal objects are significant, outlier detection algorithms may mistakenly detect low-density normal objects as outliers. Therefore, it is necessary to validate the robustness of OSD on density-imbalanced datasets. We construct a set of density-imbalanced datasets, with imbalance levels (\emph{i.e.}, the ratio of the average density of normal objects in the highest-density cluster to the average density of normal objects in the lowest-density cluster) ranging from 1 to 12. Figure \ref{fig:unbalance} shows the accuracies of the outlier detection algorithm before and after optimization on these datasets. The experimental results show that, the accuracy curve of the outlier detection algorithm optimized by OSD is always higher than the original accuracy curve. Although nearly all accuracy curves decrease as the imbalance level increases, the accuracy curve of the outlier detection algorithm optimized by OSD decreases more gradually. Therefore, OSD remains effective on density-imbalanced datasets.

\section*{Conclusion and Future Works}
\label{sec:conclusion}
In this paper, we propose a ‘generic’ optimization strategy called OSD to address the redundancy issue of outlier detection algorithms and preserve the original advantages of the optimized outlier detection algorithms. OSD first divides the dataset into several object-blocks, such that potential outliers are assigned into small object-blocks. After detonating a virtual bomb, following the principles of momentum and impulse in physics, OSD forces small object-blocks to rocket away from large object-blocks, thereby increasing the distance between outliers and normal objects. Compared to the original dataset, outlier detection algorithms can more easily distinguish the differences between outliers and normal objects in the dataset modified by OSD, making it easier to achieve higher accuracy. We have confirmed the effectiveness and robustness of OSD through extensive experiments. In terms of average accuracy, OSD improves outlier detection algorithms by an average of 15\% (AUC) and 63.7\% (AP). 

However, OSD has two input parameters ($k$, $T$) and high time complexity, which may impact the user experience. In the future, we will further improve OSD to eliminate input parameters and lower time complexity.

\ifCLASSOPTIONcaptionsoff
  \newpage
\fi



\bibliographystyle{IEEEtran}

\bibliography{OSD}

\end{document}